\documentclass[twoside]{article}

\newif\ifarxiv
\arxivtrue

\ifarxiv
\usepackage[preprint]{aistats2027}
\else
\usepackage{aistats2027}
\fi

\usepackage{graphicx}
\usepackage{amssymb}
\usepackage{mathtools}
\usepackage{amsmath}
\usepackage{amsthm}
\usepackage{nicematrix}
\usepackage{makecell}
\DeclareMathOperator*{\argmax}{arg\,max}
\DeclareMathOperator*{\argmin}{arg\,min}

\newcommand{\dataset}[1]{\textbf{#1}}
\newtheorem{proposition}{Proposition}
\usepackage{nicematrix}
\usepackage{booktabs}
\usepackage{algorithm}
\usepackage{algorithmic} 
\usepackage{subcaption}
\usepackage{url}
\usepackage{multirow}
\usepackage{bbm}
\usepackage[round]{natbib}

\ifarxiv
\usepackage[colorlinks=true,
            citecolor=blue,
            linkcolor=blue,
            urlcolor=blue]{hyperref}
\fi

\begin{document}

%

%

\twocolumn[

\aistatstitle{NPBoost: Neural Processes with Gradient-Boosted Fixed Effects}

\aistatsauthor{
Andrea Nava$^{1,2,*}$ \And
Ken Rölli$^{1}$ \And
Armin Begic$^{1}$ \And
Fabio Sigrist$^{1}$
}
\vspace{0.2cm}
\aistatsaddress{
$^{1}$Seminar for Statistics, ETH Zurich, Switzerland \\ 
$^{2}$Lucerne University of Applied Sciences and Arts, Switzerland\\
$^{*}$Corresponding author: \texttt{navaan@ethz.ch}
}
]

\begin{abstract}
  Neural Processes (NPs) are model-based meta-learners that implicitly learn a stochastic process and adapt to a new task from a small context set. Most extensions of NPs focus on improving the neural network architecture. We instead develop an extension motivated by the shared hierarchical interpretation of meta-learning and mixed-effects models. Specifically, we introduce Neural Process Boosting (\textsc{NPBoost}), which decomposes structured response variability into tree-boosted fixed effects shared across tasks and NP random effects that capture stochastic task-to-task variation. We propose to train the two components jointly using a boosting algorithm in which an NP learns residual task-specific structure and a tree ensemble estimates common patterns across tasks. Across synthetic and real-world tabular meta-learning problems, this decomposition improves over a standard NP when the shared structure contains discontinuities or other irregular patterns that boosted trees can represent effectively.
\end{abstract}

\section{Introduction}

Meta-learning studies how information shared across a collection of related tasks can be used to adapt rapidly to a new task from only a small number of labeled observations.
Neural Processes (NPs) implement this idea by learning a distribution over functions from repeated context--target prediction problems \citep{garnelo2018neuralprocesses}:
given a context set of observed input--output pairs, an NP produces a predictive distribution at new target inputs without updating its global parameters. The same context-conditioned meta-learning principle motivates recent
tabular foundation models such as TabPFN and TabICL, which are pretrained
across large collections of synthetic tabular tasks and predict on a new
dataset from its labeled examples in a single forward pass
\citep{hollmann2025accurate,qu2025tabicl}.

Most developments in the Neural Process literature have focused on modifying the neural network architecture used to represent task-specific functions. 
For example, Attentive Neural Processes introduce target-dependent context aggregation \citep{kim2019iclr-attentive}, while convolutional variants encode translation equivariance \citep{Gordon2020Convolutional}.
These approaches improve how a single NP models each task.

We investigate a different direction.
Meta-learning can be interpreted hierarchically: observations belong to tasks, while the tasks themselves are drawn from a higher-level distribution learned during meta-training \citep{baxter1998theoretical,grant2018iclr-recasting}.
In this view, training across many tasks learns information that acts as a prior when adapting to a new task.

Mixed-effects models rely on a closely related hierarchical structure.
For a given task \(i\) with \(N_i\) observations and \(p\) features, let \(X_i \in \mathbb{R}^{N_i\times p}\) and \(y_i \in \mathbb{R}^{N_i}\) be features and responses.
Mixed-effects models decompose \(y_i\) into fixed effects shared across tasks and task-specific random effects:
\begin{equation}
    y_i
    =
    \underbrace{\mathstrut F(X_i)}_{\text{fixed effects}}
    +
    \underbrace{\mathstrut b_i(X_i)}_{\text{random effects}}
    + \underbrace{\mathstrut \varepsilon_i,}_{\text{observation noise}}
\label{eq:intro_mixed_model}
\end{equation}
where \(F, b_i: \mathbb{R}^p \to \mathbb{R}\) are functions and \(F(X_i)\) and \(b_i(X_i)\) denote their row-wise evaluations on \(X_i\).
The fixed-effects function \(F\) captures systematic structure shared across tasks, while
the random effects $b_i$ are task-specific realizations drawn from a common population-level
distribution $G_{\eta}$ parametrized by ${\eta}$.
Learning $\eta$ from many observed tasks and conditioning the resulting
distribution on observations from a new task corresponds exactly to the hierarchical
interpretation of meta-learning.
Standard NPs do not explicitly separate shared fixed effects from the
task-specific stochastic functions.
Instead, the complete response variation is represented through the
context-conditioned NP. In this paper, we suggest a different approach.
Rather than increasing the expressiveness of the network architecture, we introduce an explicit decomposition between structure shared across tasks and stochastic task-specific variation.

Specifically, we introduce \textsc{NPBoost}, which models the fixed-effects function $F$ by a gradient-boosted tree ensemble and the task-specific random effects $b_i$ by an NP.
This choice is well suited for tabular meta-learning: boosted trees are strong learners of irregular, non-linear, and discontinuous shared patterns \citep{friedman2001greedy,grinsztajn2022tree}, while the NP enables uncertainty-aware adaptation to an individual task from a small context set. Representative NPBoost predictions are shown
in Figure~\ref{fig:npboost_preview}.
\begin{figure*}[t]
    \centering
    \includegraphics[
        width=\textwidth
    ]{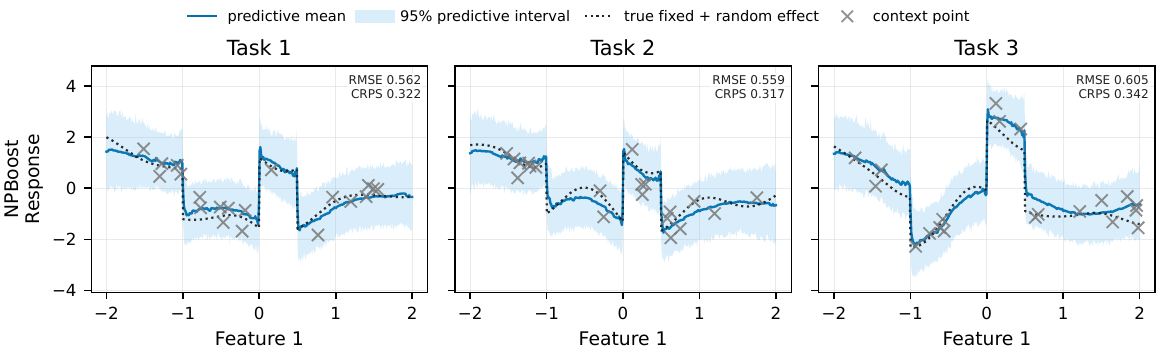}
    \caption{
        \textbf{NPBoost predictive fits in the one-dimensional reference
        setting.}
        Columns show predictions for three unseen test tasks. 
        Solid lines denote predictive means, shaded regions denote
        $95\%$ prediction intervals, and dotted lines show the true latent
        response functions $F+b_i$.
        Markers indicate the labeled observations used as context for few-shot prediction.
        The fitted predictions track the discontinuities shared across tasks
        while adapting to their smooth task-specific variation.
    }
    \label{fig:npboost_preview}
\end{figure*}
In our approach, the two components are coupled through a shared marginal predictive objective and trained jointly using a novel boosting algorithm.

In summary, our contributions are threefold: (i) we introduce an extension of Neural Processes, denoted as \textsc{NPBoost}, through a mixed-effects decomposition that separates shared structure from stochastic task-specific variation, and formulate the two components under a shared marginal predictive objective; (ii) we develop a training algorithm with context--target cross-fitting, enabling second-order boosting updates despite the context-dependent nature of NP predictions; (iii) we evaluate the proposed method on synthetic and real-world tabular meta-learning problems.

\section{Background and Related Work}
\label{s:background}

\subsection{Notation and Setup}
We consider $K$ related tasks.  
For task $i \in \{1,\dots,K\}$, we observe features
$X_i \in \mathbb{R}^{N_i \times p}$ and responses
$y_i \in \mathbb{R}^{N_i}$, where the number of observations
$N_i$ may vary across tasks. We write $x_{ij}\in\mathbb R^p$ and $y_{ij}\in\mathbb R$ for the feature
vector and response of observation $j\in\{1,\dots,N_i\}$ in task $i$,
respectively. We denote the complete dataset for task \(i\) by
\(\mathcal D_i=(X_i,y_i)\).
For each task, let $\mathcal{C}_i$ and $\mathcal{T}_i$ denote disjoint context and target index sets satisfying
$
    \mathcal{C}_i \cup \mathcal{T}_i
    =
    \{1,\dots,N_i\}
    $ and $
    \mathcal{C}_i \cap \mathcal{T}_i
    =
    \emptyset.
$
For an index set $\mathcal{A}$, we write
$(X_i^{\mathcal{A}},y_i^{\mathcal{A}})$ for the corresponding observations.
The context set
$(X_i^{\mathcal{C}_i},y_i^{\mathcal{C}_i})$
contains labeled observations from task $i$.
Given this context set and a set of target features $X_i^{\mathcal{T}_i}$,
the goal is to estimate the conditional predictive distribution of the corresponding target outputs \(y_i^{\mathcal T_i}\):
\begin{equation}
    p\!\left(
        y_i^{\mathcal{T}_i}
        \mid
        X_i^{\mathcal{T}_i},
        X_i^{\mathcal{C}_i},
        y_i^{\mathcal{C}_i}
    \right).
    \label{eq:meta_prediction_problem}
\end{equation}

We consider two prediction settings.
In \emph{within-task prediction}, the model predicts unseen observations from tasks observed during training.
In \emph{few-shot prediction}, the model predicts on unseen tasks using only a small labeled context set.

\subsection{Meta-Learning as Hierarchical Learning}
\label{ss:problem_setting}

Meta-learning can be expressed through a probabilistic hierarchical model.
Let $\tau_i$ denote a task-specific latent variable or function.
A generic formulation is
\begin{equation}
    \tau_i \overset{\mathrm{iid}}{\sim} p_{\eta}(\tau),
    \qquad
    \mathcal{D}_i \mid \tau_i
    \sim
    p_{\omega}(\mathcal{D}_i \mid \tau_i),
    \label{eq:hierarchical_meta_learning}
\end{equation}
where $\eta$ parameterizes the higher-level distribution over tasks and
$\omega$ contains parameters shared across tasks. The collection of training tasks
$\mathcal{D}_{1:K}=\{\mathcal{D}_1,\dots,\mathcal{D}_K\}$
provides repeated realizations from this hierarchy and is used to learn
the global parameters $(\omega,\eta)$. For a new task with context set
$\mathcal{D}_{\star}^{\mathcal C}
=
(X_{\star}^{\mathcal C},y_{\star}^{\mathcal C})$,
the learned higher-level distribution acts as a prior and is updated to
$
    p_{\widehat\omega,\widehat\eta}
    \left(
        \tau_{\star}
        \mid
        \mathcal{D}_{\star}^{\mathcal C}
    \right)
$.
Predictions at target inputs $X_{\star}^{\mathcal T}$ are then obtained by marginalizing the
task-specific quantity,
\begin{equation}
\begin{aligned}
    &p_{\widehat\omega,\widehat\eta}
    \left(
        y_{\star}^{\mathcal T}
        \mid
        X_{\star}^{\mathcal T},
        \mathcal{D}_{\star}^{\mathcal C}
    \right)
    \\
    &\quad =
    \int
    p_{\widehat\omega}
    \left(
        y_{\star}^{\mathcal T}
        \mid
        X_{\star}^{\mathcal T},
        \tau_{\star}
    \right)
    p_{\widehat\omega,\widehat\eta}
    \left(
        \tau_{\star}
        \mid
        \mathcal{D}_{\star}^{\mathcal C}
    \right)
    \,\mathrm{d}\tau_{\star}.
\end{aligned}
\label{eq:hierarchical_meta_prediction}
\end{equation}

Thus, meta-training learns a higher-level distribution over task-specific quantities, while the context
set performs task-specific adaptation.

Prior-data fitted networks (PFNs), now among the most prominent
model-based meta-learning approaches for tabular data, make this interpretation explicit \citep{2022iclr-transformers}. They specify a prior through a synthetic task-generating process, sample
meta-training datasets from this prior, and train a neural network to
approximate the corresponding posterior predictive distribution for a new
context set.
TabPFN applies this approach to tabular prediction
\citep{hollmann2025accurate}.

This interpretation also applies to optimization-based approaches:
\citet{grant2018iclr-recasting} show that Model-Agnostic Meta-Learning (MAML) can be
viewed as approximate hierarchical Bayesian inference, with the learned
initialization acting as a population-level prior.

\subsection{Neural Processes}
\label{ss:background_on_nps}

Neural Processes (NPs) are model-based meta-learners that use neural networks to parameterize distributions over functions
\citep{garnelo2018conditional,garnelo2018neuralprocesses}.
The Neural Process family includes conditional and latent variants \citep{dubois2020npf}.
We focus on latent NPs, which represent uncertainty about the task-specific function through a finite-dimensional global latent variable, and refer to them simply as NPs.

A standard latent NP consists of an encoder and a decoder.
The encoder is built from two neural networks, with parameter vectors  \(\phi_h,\phi_g\).
We write \(\phi=(\phi_g, \phi_h)\) for their combined parameters.
The first network \(h_{\phi_h}:\mathbb{R}^p\times\mathbb{R}\rightarrow\mathbb{R}^{d_r}\) maps every context observation to a representation
\begin{equation}
    r_{ij}
    =
    h_{\phi_h}(x_{ij},y_{ij}),
    \qquad j\in\mathcal{C}_i.
\end{equation}
A permutation-invariant aggregation, here the mean, produces a task-level summary,
\begin{equation}
    \bar r_i
    =
    \frac{1}{|\mathcal{C}_i|}
    \sum_{j\in\mathcal{C}_i} r_{ij}.
\end{equation}
The second network \(g_{\phi_g} = (g_{\phi_g,\mu},\,g_{\phi_g,v})\) with \(g_{\phi_g,\mu},\,g_{\phi_g,v}:\mathbb{R}^{d_r}\to\mathbb{R}^{d_z}\)
maps the aggregated representation to
\begin{equation}
    \mu_{\phi,i}
    =
    g_{\phi_g,\mu}(\bar r_i),
    \qquad
    \sigma^2_{\phi,i}
    =
    \exp\!\left\{g_{\phi_g,v}(\bar r_i)\right\},
    \label{eq:encoder_outputs}
\end{equation}
where \(\exp \{ \cdot \}\) is applied element-wise; this defines a Gaussian distribution with diagonal covariance for a global task latent variable $Z_i\in\mathbb{R}^{d_z}$:
\begin{equation}
    q_\phi\!\left(
        Z_i
        \mid
        X_i^{\mathcal{C}_i},
        y_i^{\mathcal{C}_i}
    \right)
    =
    \mathcal{N}\!\left(
        \mu_{\phi,i},
        \operatorname{diag}(\sigma_{\phi,i}^2)
    \right).
    \label{eq:np_latent_distribution}
\end{equation}
The decoder is built from a single neural network and a global variance
parameter \(\sigma_e^2>0\). 
We denote all decoder parameters by \(\theta\).
Conditional on a draw $z_i$, the decoder network maps each target input
$x_{ij}$ together with $z_i$ to the mean and variance of a Gaussian predictive distribution through the functions \(\mu_\theta: \mathbb{R}^p \times \mathbb{R}^{d_z} \to \mathbb{R}\) and 
\(\sigma_\theta^2: \mathbb{R}^p \times \mathbb{R}^{d_z} \to \mathbb{R}_{>0}\).
We parameterize the predictive variance as
\begin{equation}
\sigma_\theta^2(x_{ij},z_i) = v_\theta(x_{ij},z_i)+\sigma_e^2,
\label{eq:np_conditional_variance}
\end{equation}
where \(\sigma_e^2\) is shared across all tasks and observations.
Assuming conditional independence given $z_i$,
\begin{equation}
    p_\theta\!\left(
        y_i^{\mathcal{T}_i}
        \mid
        X_i^{\mathcal{T}_i},
        z_i
    \right)
    =
    \prod_{j\in\mathcal{T}_i}
    \mathcal{N}\!\left(
        y_{ij};
        \mu_\theta(x_{ij},z_i),
        \sigma_\theta^2(x_{ij},z_i)
    \right).
    \label{eq:np_conditional_distribution}
\end{equation}
The same latent realization $z_i$ is shared across all target observations from task $i$.
Although the decoder distribution is factorized conditional on $z_i$, marginalizing the shared latent variable induces dependence between observations from the same task.

The context-conditioned predictive distribution is
\begin{equation}
\begin{aligned}
    &p\!\left(
        y_i^{\mathcal{T}_i}
        \mid
        X_i^{\mathcal{T}_i},
        X_i^{\mathcal{C}_i},
        y_i^{\mathcal{C}_i}
    \right)
    \\
    &\quad =
    \int
    p_\theta\!\left(
        y_i^{\mathcal{T}_i}
        \mid
        X_i^{\mathcal{T}_i},
        z
    \right)
    q_\phi\!\left(
        z
        \mid
        X_i^{\mathcal{C}_i},
        y_i^{\mathcal{C}_i}
    \right)
    \,\mathrm{d}z.
\end{aligned}
\label{eq:np_predictive_distribution}
\end{equation}
For a continuous latent distribution, Equation~\eqref{eq:np_predictive_distribution} defines a continuous mixture of Gaussian distributions and is generally not available in closed form \citep{dubois2020npf}. Marginalizing over the latent variable can yield a
non-Gaussian predictive distribution, even though the decoder distribution
is Gaussian conditional on the latent variable. Furthermore, NPs learn an implicit distribution over task-specific functions
without specifying an explicit covariance kernel. The encoder and decoder
learn this stochastic-process representation from the collection of
training tasks.

\paragraph{Architectural extensions.} Numerous variants of NPs have been proposed to modify how
context information is encoded or processed. For example, a limitation of mean aggregation is that every target receives the same context summary.
Attentive Neural Processes address this limitation by adding a deterministic path in which cross-attention allows each target input to weight the context observations differently \citep{kim2019iclr-attentive}.
Convolutional Neural Processes encode translation equivariance and locality
\citep{Gordon2020Convolutional}.
Recurrent Neural Processes extend NPs to sequential and dynamical settings
\citep{willi2019recurrent}. Dimension-Agnostic Neural Processes accommodate tasks with varying input and output dimensions
\citep{lee2025dimension}.
Recently, Incremental Transformer NPs introduced efficient updates for streaming context observations
\citep{mortimer2026incremental}.

These methods primarily extend NPs by modifying how context information is represented and processed.
Our work instead changes the probabilistic decomposition of the response and is therefore complementary to these developments.

\raggedbottom
\subsection{Mixed-Effects Models}
\label{ss:mixed_effects}

Mixed-effects models are used for modeling hierarchically grouped data. In our setting, each group corresponds to a task.
For tasks $i=1,\dots,K$, a general formulation is
\begin{equation}
    b_i \overset{\mathrm{iid}}{\sim} G_{\eta},
    \qquad
    y_i \mid b_i, X_i
    \sim
    p_{\psi}\!\left(
        y_i \mid F(X_i), b_i
    \right),
    \label{eq:mixed_effect_hierarchy}
\end{equation}
where $F$ is a fixed-effects component shared across tasks, $b_i$ is a
task-specific random-effects component, and $G_{\eta}$ is their common
population distribution
\citep{lindley1972bayes,Bates2000ME}.
The fixed effects capture systematic population-level structure,
whereas the random effects describe task-specific deviations.

The parameters of the random-effects distribution $\eta$, the fixed-effects function $F$, and potentially additional likelihood parameters $\psi$ are commonly estimated
from the observed tasks by maximizing the marginal likelihood,
\begin{equation}
    (\widehat F,\widehat\eta,\widehat\psi)
    \in
    \argmax_{F,\eta,\psi}
    \sum_{i=1}^{K}
    \log
    \int
    p_{\psi}\!\left(
        y_i \mid F(X_i), b_i
    \right)
    \,\mathrm{d}G_{\eta}(b_i).
    \label{eq:mixed_effect_marginal_likelihood}
\end{equation}
This is an empirical-Bayes procedure
\citep{robbins1964empirical}: the higher-level prior distribution
$G_{\eta}$ is not fixed in advance, but learned from the collection
of tasks. The fitted distribution $G_{\widehat\eta}$ can subsequently be reused
as a prior for a new task.
This is directly analogous to the hierarchical interpretation of
task-based meta-learning in Section~\ref{ss:problem_setting}. Mixed-effects models can represent substantially more general dependence
structures, including nested, crossed, spatial, and temporal random
effects.
The connection considered in this paper concerns the simpler
setting in which tasks are independent realizations from a common
population distribution. Similarly, \citet{slavutsky2026robust} motivate explicit mixed-effects modeling
of environment-specific variation as an alternative to invariance learning
for zero-shot generalization to unseen environments.

\section{Neural Process Boosting}
\label{s:npboost}

\subsection{Model Formulation}
\label{ss:model_formulation}

For each task $i=1,\dots,K$, we assume the mixed-effects model
\begin{equation}
    y_i
    =
    F(X_i)
    +
    b_i(X_i)
    +
    \varepsilon_i,
    \qquad
    \varepsilon_i
    \sim
    \mathcal{N}\!\left(0,\sigma_e^2 I_{N_i}\right),
    \label{eq:mixed_effects_formulation}
\end{equation}
where  $F:\mathbb{R}^p\rightarrow\mathbb{R}$ is a
shared fixed-effects function,
$b_i: \mathbb{R}^p\rightarrow\mathbb{R}$ is a task-specific random-effects function, 
and $\varepsilon _i$ is an independent noise term.
We write $F(X_i)$ and $b_i(X_i)$ for the row-wise evaluation of $F$ and $b_i$ on $X_i$.
We represent $F$ by an ensemble of regression trees and the distribution of
$b_i$ by a latent NP. 
In contrast to conventional mixed-effects formulations, we do not impose a
centering constraint such as $\mathbb E[b_i(x)]=0$. 
We therefore treat $F$ and $b_i$ as an additive decomposition for prediction, not as separately
identifiable components.

Conditional on the task-level latent variable $Z_i=z_i$, the NP decoder
assigns a Gaussian distribution to the random effects at the target inputs,
\begin{equation}
    b_i(X_i^{\mathcal T_i})\mid z_i
    \sim
    \mathcal{N}\!\left(
        \mu_{\theta}(X_i^{\mathcal T_i}, z_i),
        V_{\theta}(X_i^{\mathcal T_i}, z_i)
    \right),
    \label{eq:conditional_random_effect}
\end{equation}
which is identical to the NP conditional decoder distribution from Equation~\eqref{eq:np_conditional_distribution},
except that it is defined for $b_i$ rather than $y_i$.
As for $F$ in Equation~\eqref{eq:mixed_effects_formulation},
$\mu_\theta(X_i^{\mathcal T_i},z_i)$ and
$v_\theta(X_i^{\mathcal T_i},z_i)$ denote the
row-wise evaluation of the decoder functions, and
$V_\theta(X_i^{\mathcal T_i},z_i)
=\operatorname{diag}\!\left(v_\theta(X_i^{\mathcal T_i},z_i)\right)$.
Since $b_i(X_i^{\mathcal T_i})$ and $\varepsilon_i^{\mathcal T_i}$ are
conditionally independent and Gaussian, integrating out the random effects
gives
\begin{equation}
\begin{aligned}
    &p_{\theta,F}\!\left(
        y_i^{\mathcal T_i}
        \mid
        X_i^{\mathcal T_i},
        z_i
    \right)
    \\
    &\quad =
    \mathcal{N}\!\left(
        y_i^{\mathcal T_i};
        F(X_i^{\mathcal T_i})+\mu_{\theta}(X_i^{\mathcal T_i}, z_i),
        \Sigma_{\theta}(X_i^{\mathcal T_i}, z_i)
    \right),
\end{aligned}
\label{eq:conditional_npboost_likelihood}
\end{equation}
with $
    \Sigma_{\theta}(X_i^{\mathcal T_i}, z_i)
    =
    V_{\theta}(X_i^{\mathcal T_i}, z_i)
    +
    \sigma_e^2 I_{|\mathcal T_i|}.
    \label{eq:conditional_npboost_covariance}
$

The task-level latent variable is subsequently marginalized using the
context-conditioned encoder distribution:
\begin{equation}
\begin{aligned}
    &p_{\theta,\phi,F}\!\left(
        y_i^{\mathcal T_i}
        \mid
        X_i^{\mathcal T_i},
        X_i^{\mathcal C_i},
        y_i^{\mathcal C_i}
    \right)
    \\
    &\quad =
    \int
    p_{\theta,F}\!\left(
        y_i^{\mathcal T_i}
        \mid
        X_i^{\mathcal T_i},
        z
    \right)
    q_{\phi,F}\!\left(
        z
        \mid
        X_i^{\mathcal C_i},
        y_i^{\mathcal C_i}
    \right)
    \,\mathrm{d}z.
\end{aligned}
\label{eq:npboost_predictive_distribution}
\end{equation}
Here, the notation $q_{\phi,F}$ indicates that the context representation
depends on the current fixed-effects function, as made explicit in
Section~\ref{ss:alternating_training}. 

\subsection{Monte Carlo Predictive Objective}
\label{ss:training_objective}

The integral in
Equation~\eqref{eq:npboost_predictive_distribution} is generally unavailable
in closed form.
Following the direct Monte Carlo approximation of the
context-conditioned predictive likelihood used for latent Neural Processes
by \citet{foong2020npml}, we approximate the integral using samples from
the encoder distribution.
For a context--target split $(\mathcal C_i,\mathcal T_i)$, we draw
\begin{equation}
    z_{i\ell}
    \sim
    q_{\phi,F}\!\left(
        Z_i
        \mid
        X_i^{\mathcal C_i},
        y_i^{\mathcal C_i}
    \right),
    \qquad
    \ell=1,\dots,L.
    \label{eq:latent_mc_samples}
\end{equation}

For each latent sample $z_{il}$, define the conditional log-likelihood
\begin{equation}
\begin{aligned}
    s_{i\ell}(F,\theta)
    =
    \log
    \mathcal{N}\!\left(
        y_i^{\mathcal T_i};
        F(X_i^{\mathcal T_i})+\mu_{\theta,i\ell},
        \Sigma_{\theta,i\ell}
    \right),
\end{aligned}
\label{eq:conditional_log_likelihood}
\end{equation}
where $
    \mu_{\theta,i\ell}
    =
    \mu_{\theta}(X_i^{\mathcal T_i}, z_{i\ell})$ and $
    \Sigma_{\theta,i\ell}
    =
    \Sigma_{\theta}(X_i^{\mathcal T_i}, z_{i\ell}).$

For a realized context--target split and latent samples, the marginal log predictive density for task $i$ is approximated by
\begin{equation}
    \widehat{\ell}_i(F,\theta,\phi)
    =
    \log\left[
        \frac{1}{L}
        \sum_{\ell=1}^{L}
        \exp\!\left\{
            s_{i\ell}(F,\theta)
        \right\}
    \right].
    \label{eq:mc_log_predictive_density}
\end{equation}

Since context--target splits and latent variables are resampled during
training, NPBoost targets the context-averaged Monte Carlo predictive risk
\begin{equation}
    \widehat{\mathcal R}_L(F,\theta,\phi)
    =
    -
    \sum_{i=1}^{K}
    \mathbb E_{\mathcal C_i}
    \mathbb E_{z_{i1:L}\mid\mathcal C_i}
    \left[
        \widehat{\ell}_i(F,\theta,\phi)
    \right],
    \label{eq:npboost_empirical_risk}
\end{equation}
where the first expectation is with respect to the context-sampling scheme
used during training and the second is with respect to the encoder
distribution in Equation~\eqref{eq:latent_mc_samples}.

In practice, stochastic optimization samples a mini-batch of tasks,
constructs a context--target split within each sampled task, and draws
$L$ latent variables per task, yielding a stochastic estimate of
Equation~\eqref{eq:npboost_empirical_risk}.

\subsection{Training using a Boosting Algorithm}
\label{ss:alternating_training}

Let the fixed-effects ensemble after boosting round $t$ be
\begin{equation}
    F^{(t)}(\cdot)
    =
    F^{(0)}(\cdot)
    +
    \nu\sum_{s=1}^{t} f^{(s)}(\cdot),
    \label{eq:boosting_ensemble}
\end{equation}
where $f^{(s)}\in\mathcal H$ are regression-tree base learners,
$\nu>0$ is the boosting learning rate, and $F^{(0)}(\cdot)\equiv0$.

\paragraph{NP update.}
Holding the current fixed-effects function $F^{(t-1)}$ fixed, define
\begin{equation}
    \widetilde y_i^{(t-1)}
    =
    y_i-F^{(t-1)}(X_i),
    \qquad i=1,\dots,K.
    \label{eq:np_training_residuals}
\end{equation}
The encoder and decoder are evaluated using these residualized responses $\widetilde y_i^{(t-1)}$, and $(\theta^{(t-1)},\phi^{(t-1)})$ are updated to
$(\theta^{(t)},\phi^{(t)})$ by stochastic gradient minimization of
Equation~\eqref{eq:npboost_empirical_risk}.
We run this NP optimization for multiple epochs.
In every epoch, we generate new context--target splits.

\paragraph{Tree update.}
We next update \(F\) using a boosting step while holding the updated NP parameters fixed. For latent sample $\ell$ and target observation $j$, define
\begin{equation}
    e_{ij\ell}^{(t)}
    =
    y_{ij}
    -
    F^{(t-1)}(x_{ij})
    -
    \mu_{\theta,ij\ell}^{(t)},
    ~~
    \left(\sigma_{ij\ell}^{(t)}\right)^2
    =
    v_{\theta,ij\ell}^{(t)}+\left(\sigma_e^{(t)}\right)^2.
    \label{eq:boosting_component_residual}
\end{equation}
The normalized likelihood weight associated with latent sample $\ell$ is
\begin{equation}
    w_{i\ell}^{(t)}
    =
    \frac{\exp\{s_{i\ell}(F^{(t-1)},\theta^{(t)})\}}
    {\sum_{q=1}^{L}\exp\{s_{iq}(F^{(t-1)},\theta^{(t)})\}}.
    \label{eq:latent_likelihood_weights}
\end{equation}
The gradient of the negative predictive log-likelihood with respect to the fixed-effects prediction is
\begin{equation}
    g_{ij}^{(t)}
    =
    -
    \sum_{\ell=1}^{L}
    w_{i\ell}^{(t)}
    \frac{e_{ij\ell}^{(t)}}{\left(\sigma_{ij\ell}^{(t)}\right)^2}.
    \label{eq:boosting_gradient}
\end{equation}
This gradient is used for constructing a tree update. Note that the sampled latent variables and NP outputs, evaluated at \(F^{(t-1)}\), are also held fixed in \eqref{eq:boosting_gradient}.
Thus, although the encoder distribution depends on \(F\) through the residualized responses, we differentiate Equation~\eqref{eq:npboost_empirical_risk} only through the explicit fixed-effects term \(F(X_i^{\mathcal T_i})\) in the target mean.

First-order gradient boosting fits the negative functional gradient
$-g_{ij}^{(t)}$, which is a weighted average of residuals relative to the
NP means. The weights $w_{i\ell}^{(t)}$ are determined by the joint target
likelihood of latent realization $\ell$, while
$1/(\sigma_{ij\ell}^{(t)})^2$ weights by predictive precision.
Since $w_{i\ell}^{(t)}$ depends on all target observations in task $i$, the
gradient for observation $j$ also depends on the other target observations.
Thus, through the shared task-level latent variable $z_i$, the boosting
updates inherit the within-task dependence modeled by the NP rather than
being based on independent observation-wise residuals.

For Newton boosting, we additionally use the diagonal Hessian
$h_{ij}^{(t)}$ of the negative predictive log-likelihood. Its expression and derivation are provided in
Appendix~\ref{app:hessian}.

Given the gradient and diagonal-Hessian vectors, the next tree is
fitted as
\begin{equation}
    f^{(t)}
    \in
    \argmin_{f\in\mathcal H}
    \sum_{i=1}^{K}
    \sum_{j=1}^{N_i}
    h_{ij}^{(t)}
    \left(
        -\frac{g_{ij}^{(t)}}{h_{ij}^{(t)}}
        -
        f(x_{ij})
    \right)^2,
    \label{eq:newton_tree_objective}
\end{equation}
and the fixed-effects ensemble is updated according to
\begin{equation}
    F^{(t)}(\cdot)
    =
    F^{(t-1)}(\cdot)
    +
    \nu f^{(t)}(\cdot).
    \label{eq:tree_ensemble_update}
\end{equation}

A difficulty in alternating the NP and tree updates is that NP
predictions depend on the context set, so a single context--target
split provides derivatives only for its target observations. We thus
use multiple splits to obtain derivatives for all training observations.

\subsection{Context--Target Cross-Fitting} 
\label{ss:context_target_gradient}

For task $i$, let $n_{c,i}$ denote the NP training context size.
Assuming for simplicity that $n_{c,i}$ divides $N_i$, randomly partition
$\{1,\dots,N_i\}$ into $B_i = N_i/n_{c,i}$ disjoint folds
$\mathcal S_{i}^{(1)},\dots,\mathcal S_{i}^{(B_i)}$,
each containing $n_{c,i}$ observations.
For $b=1,\dots,B_i$, define the context--target split
\begin{equation}
    \mathcal C_{i}^{(b)}=\mathcal S_{i}^{(b)},
    \qquad
    \mathcal T_{i}^{(b)}
    =
    \{1,\dots,N_i\}\setminus\mathcal S_{i}^{(b)}.
    \label{eq:context_target_cross_fitting}
\end{equation}
Thus, each fold is used once as the context set and its complement as the
target set.

Let $g_{ij}^{(b)}$ and $h_{ij}^{(b)}$ denote the gradient and
diagonal-Hessian contributions for observation $j$ obtained from split $b$.
Since every observation belongs to exactly one context fold, it appears in
the target set of the remaining $B_i-1$ splits.
We therefore assemble the derivatives as
\begin{equation}
    g_{ij}
    =
    \frac{1}{B_i-1}
    \sum_{b:\,j\in\mathcal T_{i}^{(b)}}
    g_{ij}^{(b)},
    \qquad
    h_{ij}
    =
    \frac{1}{B_i-1}
    \sum_{b:\,j\in\mathcal T_{i}^{(b)}}
    h_{ij}^{(b)}.
    \label{eq:cross_fitted_derivatives}
\end{equation}

If $n_{c,i}$ does not divide $N_i$, the final context fold contains the
remaining observations. In this case, each derivative is normalized by the
number of splits in which the corresponding observation appears as a target.

\begin{proposition}[Unbiased cross-fitted derivatives]
\label{prop:crossfit_unbiased}
Assume that the folds are obtained from a uniformly random partition.
For any observation $j$, let $g_{ij}(\mathcal C_i)$ denote the gradient
contribution obtained when $\mathcal C_i$ is used as context and $j$ belongs
to the corresponding target set. Then the cross-fitted gradient in
Equation~\eqref{eq:cross_fitted_derivatives} satisfies
\begin{equation}
    \mathbb E[g_{ij}]
    =
    \mathbb E\!\left[
        g_{ij}(\mathcal C_i)
        \,\middle|\,
        |\mathcal C_i|=n_{c,i},\;
        j\notin\mathcal C_i
    \right].
\end{equation}
On the left, the expectation is over the random partition and the latent draws; on the right,
$\mathcal C_i$ is uniformly distributed over all subsets of size $n_{c,i}$
that do not contain $j$. The same holds for the diagonal
Hessian.
\end{proposition}

Proposition~\ref{prop:crossfit_unbiased} shows that the cross-fitted
gradient is an unbiased estimator of the gradient contribution for
observation $j$, averaged over context sets of the chosen size.
Importantly, by choosing $n_{c,i}$ equal to the context size used during NP training,
this averaging is taken with respect to the same context-sampling
distribution that enters the predictive risk in
Equation~\eqref{eq:npboost_empirical_risk}.
A proof is provided in Appendix~\ref{app:proof_prop1}.
Finally, the complete training algorithm is given in Appendix~\ref{app:npboost_training}.

\subsection{Prediction}
\label{ss:prediction}

Given fitted parameters
$(\widehat F,\widehat\theta,\widehat\phi)$ and a context set
$(X_i^{\mathcal C_i},y_i^{\mathcal C_i})$, we draw
\begin{equation}
    z_{i\ell}
    \sim
    q_{\widehat\phi,\widehat F}\!\left(
        Z_i
        \mid
        X_i^{\mathcal C_i},
        y_i^{\mathcal C_i}
    \right),
    \qquad
    \ell=1,\dots,L.
\end{equation}

For each latent sample, the predictive distribution at the target inputs is
Gaussian with diagonal covariance. 
The full predictive distribution is approximated by an equally weighted
Gaussian mixture
\begin{equation}
    \frac{1}{L}
    \sum_{\ell=1}^{L}
    \mathcal N\!\left(
        y_i^{\mathcal T_i};
        \widehat F(X_i^{\mathcal T_i})
        +
        \mu_{\widehat\theta}(X_i^{\mathcal T_i}, z_{i\ell}),
        \Sigma_{\widehat\theta}(X_i^{\mathcal T_i}, z_{i\ell})
    \right).
\label{eq:npboost_prediction}
\end{equation}

For point prediction, we use the mean of this mixture.
To obtain samples from the full predictive distribution, we first draw the
latents as above and then sample from the corresponding conditional
Gaussians.
This captures uncertainty both over the task-level latent
variable and conditional on each latent realization.

\subsection{Attention-Based Extension}
\label{ss:anpboost}

We also consider an attentive variant of NPBoost, denoted
\textsc{ANPBoost}, in which the standard NP component is
replaced by an Attentive Neural Process (ANP) \citep{kim2019iclr-attentive}.
The attentive implementation augments the latent path with a deterministic
path based on multi-head cross-attention, using eight attention heads.
The model formulation, Monte Carlo objective, alternating optimization,
context--target cross-fitting, and prediction procedure remain
unchanged. Further details are presented in Appendix \ref{app:implementation_details}.

\section{Experiments}
\label{s:experiments}

We evaluate NPBoost and ANPBoost on synthetic and real-world tabular datasets.
Our experiments address three questions: (i) can the mixed-effects decomposition improve over a standard NP when tasks share structure? (ii) when does cross-attention improve NP and NPBoost? (iii) how does NPBoost compare with mixed-model and gradient-boosted tree baselines?
\subsection{Prediction and Evaluation Protocol} 
\label{ss:prediction_scenarios}
We consider two prediction scenarios and train and select a separate model for each.
\emph{Within-task prediction} measures generalization to held-out observations from tasks seen during training. 
\emph{Few-shot prediction} measures adaptation to previously unseen tasks from a small labeled context set.
For within-task prediction, we split the observations of every task into
training, validation, and test subsets containing $50\%$, $25\%$, and
$25\%$ of the observations, respectively.
For few-shot prediction, we instead partition tasks into disjoint training,
validation, and test sets. Each unseen validation or test task provides a
small labeled context set, with predictions evaluated on the remaining
observations.

We evaluate point predictions using the RMSE and probabilistic predictions using the CRPS \citep{gneiting2007strictly}.
For NP-based models, the CRPS is estimated from samples from the predictive
distribution. Further details are provided in Appendix~\ref{app:evaluation_criteria}.

We compare NPBoost and ANPBoost with the standard and attentive NP baselines, and with gradient-boosted
trees including the task identifier as a categorical feature (Task-ID GBT)
\citep{ke2017lightgbm}. We further consider a Gaussian
process model with linear fixed effects (GPLinear) and a linear mixed-effects
model with random intercepts and slopes (LME) \citep{Bates2000ME}.
Finally, we compare with TabICLv2 \citep{qu2026tabiclv2} in two
configurations: TabICL\_T conditions only on observations from the current
task, whereas TabICL\_P pools observations across tasks and includes the task
identifier as a categorical feature.
Implementation details and hyperparameter tuning with early
stopping are described in
Appendices~\ref{app:implementation_details} and~\ref{app:parameter_tuning}. 

\subsection{Synthetic Datasets}
\label{ss:synthetic_datasets}

We generate data from Equation~\eqref{eq:mixed_effects_formulation}. 
The inputs are sampled independently as
$x_{ij}\sim\operatorname{Unif}([-2,2]^d)$, with $d\in\{1,2\}$.
The shared fixed-effects function $F$ is piecewise constant, with
four regions for $d=1$ and seven regions for $d=2$, and is scaled such that
$\operatorname{Var}\!\left(F(x_{ij})\right)=1.$
The task-specific effects are independent draws from a zero-mean Gaussian process with RBF kernel, and the observation noise satisfies
    $\varepsilon_{ij}\sim\mathcal N(0,0.25)$.
Each dataset contains $K=600$ tasks with $N_i = 100$ observations per task.
For few-shot prediction, each validation and test task provides $20$ labeled context points.
The data-generation and splitting procedure is independently repeated for five seeds.

As an ablation, we consider a setting in which the fixed-effects component is
removed. Since the task-specific functions are Gaussian processes, the resulting data-generating model
corresponds to the standard GP regression benchmark commonly used in the
Neural Process literature \citep{kim2019iclr-attentive, mortimer2026incremental}.
Complete results are reported in Appendix~\ref{app:complete_results}.
We additionally consider robustness experiments that vary the number of
tasks, the smoothness of the task-specific process, the observation-noise
distribution, and the presence of irrelevant features. More details are provided in Appendices~\ref{app:robust}
and~\ref{app:synthetic_data_generation}.

\subsection{Real-World Datasets}
\label{ss:real_datasets}

We evaluate the methods on four real-world datasets from different domains: \dataset{cars}~\citep{reese2021cars}, \dataset{Spotify}~\citep{dslc2024tidytuesday}, \dataset{cows}~\citep{diggle2002analysis}, and \dataset{bikes}~\citep{uci2020seoulbike}.
Complete feature definitions, preprocessing steps, and data-splitting procedures are provided in Appendix~\ref{app:real_world_datasets}.

\section{Results and Discussion}
\label{s:results} 

\subsection{Results for Synthetic Datasets}
\label{ss:results_synthetic_data}

Table~\ref{tab:synthetic_results_main} reports results for the reference
synthetic experiment. Complete results for all synthetic settings and
additional baselines are provided in Appendix~\ref{app:complete_results}.

\begin{table*}[ht!]
\centering
\caption{\textbf{Predictive performance on the synthetic reference experiment.} Each entry reports the mean RMSE / mean CRPS over \(5\) seeds, with standard errors shown in parentheses below in the same order. The lowest mean RMSE and CRPS in each row are shown in bold, and \textemdash{} indicates an unavailable metric.}
\label{tab:synthetic_results_main}
\small
\resizebox{\textwidth}{!}{%
\begin{NiceTabular}{ll *{8}{c}}
\toprule
\textbf{\(d\)} & \textbf{Scenario}
& \textbf{NPBoost}
& \textbf{ANPBoost}
& \textbf{NP}
& \textbf{ANP}
& \textbf{Task-ID GBT}
& \textbf{GPLinear}
& \textbf{TabICL\_T}
& \textbf{TabICL\_P} \\
\midrule

\Block{2-1}{1} & within-task
& \makecell[cc]{\textbf{0.576} / \textbf{0.325}\\{\scriptsize (0.001 / 0.001)}}
& \makecell[cc]{\textbf{0.576} / \textbf{0.325}\\{\scriptsize (0.001 / 0.001)}}
& \makecell[cc]{0.596 / 0.335\\{\scriptsize (0.001 / 0.001)}}
& \makecell[cc]{\textbf{0.576} / 0.326\\{\scriptsize (0.001 / 0.001)}}
& \makecell[cc]{0.594 / \textemdash{}\\{\scriptsize (0.002 / \textemdash{})}}
& \makecell[cc]{0.685 / 0.379\\{\scriptsize (0.002 / 0.001)}}
& \makecell[cc]{0.703 / 0.391\\{\scriptsize (0.003 / 0.001)}}
& \makecell[cc]{1.104 / 0.621\\{\scriptsize (0.004 / 0.003)}} \\

& few-shot
& \makecell[cc]{0.688 / 0.387\\{\scriptsize (0.010 / 0.006)}}
& \makecell[cc]{0.659 / 0.371\\{\scriptsize (0.008 / 0.004)}}
& \makecell[cc]{0.700 / 0.393\\{\scriptsize (0.008 / 0.004)}}
& \makecell[cc]{\textbf{0.658} / \textbf{0.370}\\{\scriptsize (0.009 / 0.005)}}
& \makecell[cc]{1.118 / \textemdash{}\\{\scriptsize (0.019 / \textemdash{})}}
& \makecell[cc]{0.836 / 0.455\\{\scriptsize (0.005 / 0.003)}}
& \makecell[cc]{0.912 / 0.508\\{\scriptsize (0.007 / 0.003)}}
& \makecell[cc]{1.137 / 0.641\\{\scriptsize (0.020 / 0.011)}} \\

\midrule

\Block{2-1}{2} & within-task
& \makecell[cc]{0.862 / 0.488\\{\scriptsize (0.003 / 0.002)}}
& \makecell[cc]{\textbf{0.860} / \textbf{0.484}\\{\scriptsize (0.002 / 0.001)}}
& \makecell[cc]{0.906 / 0.519\\{\scriptsize (0.002 / 0.004)}}
& \makecell[cc]{0.866 / 0.486\\{\scriptsize (0.002 / 0.001)}}
& \makecell[cc]{0.934 / \textemdash{}\\{\scriptsize (0.004 / \textemdash{})}}
& \makecell[cc]{0.953 / 0.527\\{\scriptsize (0.003 / 0.001)}}
& \makecell[cc]{1.091 / 0.610\\{\scriptsize (0.005 / 0.003)}}
& \makecell[cc]{1.125 / 0.634\\{\scriptsize (0.003 / 0.002)}} \\

& few-shot
& \makecell[cc]{1.015 / 0.572\\{\scriptsize (0.006 / 0.004)}}
& \makecell[cc]{0.973 / \textbf{0.546}\\{\scriptsize (0.006 / 0.004)}}
& \makecell[cc]{1.031 / 0.582\\{\scriptsize (0.008 / 0.005)}}
& \makecell[cc]{\textbf{0.972} / 0.547\\{\scriptsize (0.006 / 0.004)}}
& \makecell[cc]{1.118 / \textemdash{}\\{\scriptsize (0.009 / \textemdash{})}}
& \makecell[cc]{1.182 / 0.653\\{\scriptsize (0.002 / 0.001)}}
& \makecell[cc]{1.369 / 0.785\\{\scriptsize (0.003 / 0.002)}}
& \makecell[cc]{1.127 / 0.635\\{\scriptsize (0.008 / 0.004)}} \\

\bottomrule
\end{NiceTabular}%
}
\end{table*}

\paragraph{The mixed-effects decomposition consistently improves the standard NP.}
Across all synthetic settings, NPBoost matches or improves upon NP in
RMSE and CRPS under both within-task and few-shot prediction. 

\paragraph{Attention closes the gap to NPBoost.}
ANP consistently improves over the standard NP and often performs similarly
to, or slightly better than, NPBoost. Adding the boosted fixed effect to the
attentive model can still yield gains, but they are smaller. 

\paragraph{Generic tabular and linear mixed-effects baselines are less competitive.}

GPLinear and LME perform substantially worse (Appendix~\ref{app:complete_results}) than the flexible NP-based models. 
The TabICL models are also not competitive in this environment: neither variant matches the NP-based models.
Task-ID GBT performs reasonably for some within-task problems, but its performance drops sharply in the few-shot setting. 

\subsection{Results for Real-World Datasets}
\label{ss:results_real_world}

Table~\ref{tab:real_results_main} summarizes the main comparisons across
the real-world benchmarks. Complete results, including additional baselines,
are provided in Appendix~\ref{app:complete_results}.

\begin{table*}[ht!]
\centering
\caption{\textbf{Predictive performance across real-world datasets.}
Each entry reports mean RMSE / mean CRPS over \(5\) seeds, with standard
errors in parentheses. The lowest mean RMSE and CRPS in each row are shown
in bold, and \textemdash{} indicates an unavailable metric.}
\label{tab:real_results_main}
\small
\resizebox{\textwidth}{!}{%
\begin{NiceTabular}{ll *{7}{c}}
\toprule
\textbf{Dataset} & \textbf{Scenario}
& \textbf{NPBoost}
& \textbf{ANPBoost}
& \textbf{NP}
& \textbf{ANP}
& \textbf{Task-ID GBT}
& \textbf{GPLinear}
& \textbf{TabICL\_P} \\
\midrule

\Block{2-1}{\dataset{cars}} & within-task
& \makecell[cc]{\textbf{0.257} / 0.127\\{\scriptsize (0.002 / 0.001)}}
& \makecell[cc]{0.258 / 0.134\\{\scriptsize (0.003 / 0.006)}}
& \makecell[cc]{0.266 / \textbf{0.125}\\{\scriptsize (0.003 / 0.001)}}
& \makecell[cc]{0.270 / 0.128\\{\scriptsize (0.003 / 0.001)}}
& \makecell[cc]{0.271 / \textemdash{}\\{\scriptsize (0.003 / \textemdash{})}}
& \makecell[cc]{0.297 / 0.138\\{\scriptsize (0.003 / 0.001)}}
& \makecell[cc]{0.316 / 0.151\\{\scriptsize (0.002 / 0.001)}} \\

& few-shot
& \makecell[cc]{0.272 / 0.142\\{\scriptsize (0.010 / 0.008)}}
& \makecell[cc]{\textbf{0.265} / \textbf{0.136}\\{\scriptsize (0.008 / 0.002)}}
& \makecell[cc]{0.281 / \textbf{0.136}\\{\scriptsize (0.009 / 0.004)}}
& \makecell[cc]{0.288 / 0.143\\{\scriptsize (0.012 / 0.003)}}
& \makecell[cc]{0.401 / \textemdash{}\\{\scriptsize (0.018 / \textemdash{})}}
& \makecell[cc]{0.368 / 0.177\\{\scriptsize (0.023 / 0.008)}}
& \makecell[cc]{0.387 / 0.206\\{\scriptsize (0.013 / 0.008)}} \\

\midrule

\Block{2-1}{\dataset{Spotify}} & within-task
& \makecell[cc]{0.0907 / 0.1393\\{\scriptsize (0.0007 / 0.0235)}}
& \makecell[cc]{0.0912 / 0.1662\\{\scriptsize (0.0006 / 0.0120)}}
& \makecell[cc]{0.0922 / 0.1242\\{\scriptsize (0.0007 / 0.0172)}}
& \makecell[cc]{0.0928 / 0.1295\\{\scriptsize (0.0007 / 0.0242)}}
& \makecell[cc]{0.0928 / \textemdash{}\\{\scriptsize (0.0008 / \textemdash{})}}
& \makecell[cc]{0.1012 / \textemdash{}\\{\scriptsize (0.0005 / \textemdash{})}}
& \makecell[cc]{\textbf{0.0874} / \textbf{0.0447}\\{\scriptsize (0.0006 / 0.0003)}} \\

& few-shot
& \makecell[cc]{0.0979 / 0.2032\\{\scriptsize (0.0043 / 0.0034)}}
& \makecell[cc]{0.0998 / 0.2093\\{\scriptsize (0.0047 / 0.0026)}}
& \makecell[cc]{0.0998 / 0.1630\\{\scriptsize (0.0043 / 0.0085)}}
& \makecell[cc]{0.1012 / 0.1738\\{\scriptsize (0.0044 / 0.0189)}}
& \makecell[cc]{0.1240 / \textemdash{}\\{\scriptsize (0.0042 / \textemdash{})}}
& \makecell[cc]{0.1102 / 0.0590\\{\scriptsize (0.0040 / 0.0022)}}
& \makecell[cc]{\textbf{0.0909} / \textbf{0.0480}\\{\scriptsize (0.0032 / 0.0018)}} \\

\midrule

\Block{2-1}{\dataset{cows}} & within-task
& \makecell[cc]{\textbf{0.232} / 0.161\\{\scriptsize (0.007 / 0.018)}}
& \makecell[cc]{0.234 / 0.176\\{\scriptsize (0.010 / 0.015)}}
& \makecell[cc]{0.245 / 0.137\\{\scriptsize (0.006 / 0.004)}}
& \makecell[cc]{0.245 / 0.165\\{\scriptsize (0.006 / 0.019)}}
& \makecell[cc]{0.266 / \textemdash{}\\{\scriptsize (0.011 / \textemdash{})}}
& \makecell[cc]{0.239 / \textbf{0.130}\\{\scriptsize (0.005 / 0.002)}}
& \makecell[cc]{0.239 / \textbf{0.130}\\{\scriptsize (0.006 / 0.003)}} \\

& few-shot
& \makecell[cc]{0.258 / 0.167\\{\scriptsize (0.014 / 0.012)}}
& \makecell[cc]{0.249 / 0.173\\{\scriptsize (0.011 / 0.010)}}
& \makecell[cc]{0.270 / 0.174\\{\scriptsize (0.014 / 0.013)}}
& \makecell[cc]{0.254 / 0.191\\{\scriptsize (0.011 / 0.010)}}
& \makecell[cc]{0.317 / \textemdash{}\\{\scriptsize (0.011 / \textemdash{})}}
& \makecell[cc]{0.253 / 0.140\\{\scriptsize (0.014 / 0.008)}}
& \makecell[cc]{\textbf{0.236} / \textbf{0.131}\\{\scriptsize (0.012 / 0.006)}} \\

\midrule

\Block{2-1}{\dataset{bikes}} & within-task
& \makecell[cc]{0.287 / 0.129\\{\scriptsize (0.009 / 0.008)}}
& \makecell[cc]{\textbf{0.282} / 0.124\\{\scriptsize (0.009 / 0.006)}}
& \makecell[cc]{0.294 / 0.121\\{\scriptsize (0.008 / 0.003)}}
& \makecell[cc]{0.296 / 0.124\\{\scriptsize (0.010 / 0.006)}}
& \makecell[cc]{0.302 / \textemdash{}\\{\scriptsize (0.006 / \textemdash{})}}
& \makecell[cc]{0.429 / 0.200\\{\scriptsize (0.005 / 0.002)}}
& \makecell[cc]{0.309 / \textbf{0.108}\\{\scriptsize (0.006 / 0.001)}} \\

& few-shot
& \makecell[cc]{0.588 / 0.340\\{\scriptsize (0.070 / 0.032)}}
& \makecell[cc]{0.550 / 0.322\\{\scriptsize (0.058 / 0.022)}}
& \makecell[cc]{0.577 / 0.316\\{\scriptsize (0.063 / 0.036)}}
& \makecell[cc]{0.551 / 0.292\\{\scriptsize (0.054 / 0.036)}}
& \makecell[cc]{0.706 / \textemdash{}\\{\scriptsize (0.091 / \textemdash{})}}
& \makecell[cc]{0.646 / 0.325\\{\scriptsize (0.086 / 0.040)}}
& \makecell[cc]{\textbf{0.351} / \textbf{0.151}\\{\scriptsize (0.028 / 0.010)}} \\

\bottomrule
\end{NiceTabular}%
}
\end{table*}

\paragraph{NPBoost improves point prediction over the NP.}
NPBoost achieves lower RMSE than NP in seven of the eight
dataset--scenario combinations. NPBoost also outperforms both tree baselines (Appendix~\ref{app:complete_results}) and LME throughout, and GPLinear in seven of the eight comparisons.

\paragraph{Attention provides no consistent advantage.}
ANPBoost and ANP do not consistently improve upon their non-attentive counterparts on the real-world datasets.
One possible explanation is the small context size: to retain tasks with as few as $10$ observations, we use relatively small context sets.
This can reduce the benefit of cross-attention.

\paragraph{NPBoost and pooled TabICL show complementary strengths.}
In terms of RMSE, NPBoost or ANPBoost achieves the best performance in four
of the eight comparisons, while TabICL\_P performs best in the
remaining four. The NPBoost models are particularly competitive on the
largest dataset, \dataset{cars}, where they perform best in both prediction settings,
whereas TabICL\_P performs especially well on the smaller datasets and
in several few-shot settings. In contrast, TabICL\_T has
higher RMSE than NPBoost in all settings. This suggests that, when only a
small task-specific context set is available, TabICL can benefit substantially
from pooling observations across tasks.

\section{Conclusion}
\label{s:conclusion}

We introduced NPBoost, an extension of Neural Processes for tabular meta-learning.
It separates prediction into a shared tree-ensemble component and a task-specific residual component modeled by a latent NP.
The two parts are coupled through a common predictive Monte Carlo objective and trained by alternating NP and boosting updates.
Instead of modifying the NP architecture, we extend NPs by decomposing the model into shared and task-specific components.

Across synthetic experiments, NPBoost achieved lower RMSE and CRPS than the standard NP in most evaluated settings.
On the real-world datasets it reduced RMSE in seven of eight comparisons, although the improvement did not consistently extend to CRPS.
This may partly reflect our model-selection procedure, which uses validation RMSE rather than CRPS as the selection criterion. In the synthetic experiments, cross-attention generally improved both
model classes. ANP outperformed NP across all settings, while ANPBoost
generally matched or improved upon NPBoost.

\subsection*{AI use statement}

In this work, we used generative AI tools to provide feedback on research
methodology and experimental design, and to assist in the interpretation of
results. Additionally, we used generative AI tools for brainstorming, identifying
and summarizing relevant literature, and drafting and editing parts of the
manuscript. All AI-assisted suggestions, interpretations, literature
references, and manuscript edits were reviewed and verified by the authors.
The experiments and reported results were produced and checked by the authors.
We take responsibility for the final content of this work, including text,
claims, or artifacts produced with the aid of generative AI.


\bibliography{references}


\clearpage
\appendix
\onecolumn
\section{Proof of Proposition~\ref{prop:crossfit_unbiased}}
\label{app:proof_prop1}

\begin{proof}
For a fixed context set $\mathcal C_i$, define the latent-averaged
gradient $
    \bar g_{ij}(\mathcal C_i)
    =
    \mathbb E_{\mathbf Z_i\mid\mathcal C_i}
    \left[
        g_{ij}(\mathcal C_i,\mathbf Z_i)
    \right]$.
Thus, $\bar g_{ij}(\mathcal C_i)$ is the expected split-specific
gradient if the context set is held fixed and the $L$ latent variables
are repeatedly resampled from the corresponding encoder distribution. Let $\mathcal P_i$ denote the random partition into
$\mathcal C_i^{(1)},\ldots,\mathcal C_i^{(B_i)}$.
Conditional on $\mathcal P_i$, the context sets are fixed, and the only
remaining randomness in the split-specific gradients comes from the
latent draws. Therefore, by linearity of expectation,
\[
\begin{aligned}
    \mathbb E\!\left[g_{ij}\mid\mathcal P_i\right]
    &=
    \frac{1}{B_i-1}
    \sum_{b=1}^{B_i}
    \mathbf 1\!\left\{
        j\notin\mathcal C_i^{(b)}
    \right\}
    \bar g_{ij}\!\left(\mathcal C_i^{(b)}\right).
\end{aligned}
\]

We now take expectation over the random partition. Since the partition
is uniform, each fold $\mathcal C_i^{(b)}$ is marginally uniform over
all subsets of size $n_{c,i}$. Moreover,
\[
    \Pr\!\left(j\notin\mathcal C_i^{(b)}\right)
    =
    1-\frac{n_{c,i}}{N_i}
    =
    \frac{B_i-1}{B_i},
\]
and, conditional on $j\notin\mathcal C_i^{(b)}$,
$\mathcal C_i^{(b)}$ is uniformly distributed over all subsets of size
$n_{c,i}$ that do not contain $j$. Hence
\[
\begin{aligned}
    \mathbb E[g_{ij}]
    &=
    \frac{1}{B_i-1}
    \sum_{b=1}^{B_i}
    \mathbb E\!\left[
        \mathbf 1\!\left\{
            j\notin\mathcal C_i^{(b)}
        \right\}
        \bar g_{ij}\!\left(\mathcal C_i^{(b)}\right)
    \right]
    \\
    &=
    \frac{1}{B_i-1}
    \sum_{b=1}^{B_i}
    \frac{B_i-1}{B_i}
    \mathbb E\!\left[
        \bar g_{ij}(\mathcal C_i)
        \,\middle|\,
        |\mathcal C_i|=n_{c,i},\;
        j\notin\mathcal C_i
    \right]
    =
    \mathbb E\!\left[
        \bar g_{ij}(\mathcal C_i)
        \,\middle|\,
        |\mathcal C_i|=n_{c,i},\;
        j\notin\mathcal C_i
    \right].
\end{aligned}
\]

Finally, by the law of iterated expectation,
\[
\begin{aligned}
    \mathbb E\!\left[
        \bar g_{ij}(\mathcal C_i)
        \,\middle|\,
        |\mathcal C_i|=n_{c,i},\;
        j\notin\mathcal C_i
    \right]
    &=
    \mathbb E\!\left[
        g_{ij}(\mathcal C_i,\mathbf Z_i)
        \,\middle|\,
        |\mathcal C_i|=n_{c,i},\;
        j\notin\mathcal C_i
    \right],
\end{aligned}
\]
which proves the claim. The argument for the diagonal Hessian is identical.
\end{proof}

\section{Derivation of Boosting Gradient and Hessian}
\label{app:hessian}
For notational simplicity, we assume that all tasks
have the same number of observations $N_i=N$ and use the same context
size $n_c$, so that $B_i=B=N/n_c$ for all tasks. The derivation is unchanged when task sizes vary.
Consider boosting round \(t\), task \(i\), and context-target split \((\mathcal{C}_i^{(b)},\mathcal{T}_i^{(b)})\).
The risk in Equation~\eqref{eq:npboost_empirical_risk} is an expectation
over context--target splits and latent draws. During training we work
with one realized split:
\begin{equation}
    \widehat{\mathcal R}^{(b)}(F,\theta,\phi)
    = - \sum_{i=1}^{K} \widehat\ell_i^{(b)}(F,\theta,\phi),
    \label{eq:realized_risk}
\end{equation}
for the risk evaluated on split $b$ with the drawn latent samples.
For a target observation with index \(j\in\mathcal T_i^{(b)}\), we derive the gradient
and diagonal-Hessian contributions of \(\widehat {\mathcal R}^{(b)}\) with respect to \(F(x_{ij})\), evaluated at \(F^{(t-1)}\):
\begin{equation}
    g_{ij}^{(t,b)}
    =
    \left.
    \frac{\partial \widehat{\mathcal R}^{(b)} (F,\theta,\phi)}{\partial F(x_{ij})}
    \right|_{F = F^{(t-1)}}
    ,
    \qquad
    h_{ij}^{(t,b)}
    =
    \left.
    \frac{\partial^2 \widehat{\mathcal R}^{(b)}(F,\theta,\phi)}{\partial F(x_{ij})^2}
    \right|_{F = F^{(t-1)}}.
    \label{eq:boosting_grad_hess_def}
\end{equation}

During the tree update, the updated NP parameters, sampled latent variables,
context-dependent representations, and decoder outputs are held fixed.
Thus, the derivatives below do not propagate through the encoder or decoder.

For the split-specific versions of the quantities in
Equation~\eqref{eq:boosting_component_residual}, write
\begin{equation}
    e_{ij\ell}^{(t,b)}
    =
    y_{ij}
    -
    F^{(t-1)}(x_{ij})
    -
    \mu_{\theta,ij\ell}^{(t,b)},
    \qquad
    \left(\sigma_{ij\ell}^{(t,b)}\right)^2
    =
    v_{\theta,ij\ell}^{(t,b)}
    +
    \left(\sigma_e^{(t)}\right)^2.
    \label{eq:split_boosting_component_residual}
\end{equation}
Note that \(\sigma_e^{(t)}\) does not depend on split \(b\), since it is an estimated parameter of the NP. Under the factorized Gaussian decoder, the conditional log-likelihood
associated with latent sample \(\ell\) is
\begin{equation}
\begin{aligned}
    s_{i\ell}^{(t,b)}
    &=
    \sum_{m\in\mathcal T_i^{(b)}}
    \log
    \mathcal N\!\left(
        y_{im};
        F^{(t-1)}(x_{im})
        +
        \mu_{\theta,im\ell}^{(t,b)},
        \left(\sigma_{im\ell}^{(t,b)}\right)^2
    \right).
\end{aligned}
\label{eq:split_conditional_log_likelihood}
\end{equation}
This is the split-specific form of
Equation~\eqref{eq:conditional_log_likelihood}.
The corresponding normalized likelihood weight is
\begin{equation}
    w_{i\ell}^{(t,b)}
    =
    \frac{
        \exp\!\left\{s_{i\ell}^{(t,b)}\right\}
    }{
        \sum_{q=1}^{L}
        \exp\!\left\{s_{iq}^{(t,b)}\right\}
    }.
    \label{eq:split_latent_likelihood_weights}
\end{equation}

For readability, we suppress the fixed indices \(t\) and \(b\) in the remainder of the derivation.
We compute \(g_{ij}\) and \(h_{ij}\) by carefully applying the chain rule.

Since \(\widehat{\mathcal R}(F,\theta,\phi)\) is a sum over multiple tasks, differentiating with respect to \(F(x_{ij})\) only affects the contribution from task \(i\). 
Hence,
\begin{equation}
\begin{aligned}
    \frac{\partial \widehat{\mathcal R}(F,\theta,\phi)}{\partial F(x_{ij})} 
    &= \frac{\partial }{\partial F(x_{ij})} \left[ -\hat \ell_i(F, \theta, \phi)\right]\\
    &= - \frac{\partial}{\partial F(x_{ij})} \left[\log\left(
        \frac{1}{L}
        \sum_{\ell=1}^{L}
        \exp\!\left\{
            s_{i\ell}(F,\theta)
        \right\}
    \right) \right].
\end{aligned}
\end{equation}

Since $F(x_{ij})$ only enters the term with index \(m=j\) in Equation~\eqref{eq:split_conditional_log_likelihood}, we have
\begin{equation}
\frac{\partial s_{i\ell}}{\partial F(x_{ij})} = \frac{\partial }{\partial F(x_{ij})} \log
    \mathcal N\!\left(
        y_{ij};
        F(x_{ij})
        +
        \mu_{\theta,ij\ell}^{},
        \left(\sigma_{ij\ell}\right)^2
    \right) \ = \frac{e_{ij\ell}} {\left(\sigma_{ij\ell}\right)^2}.
\label{eq:derivative_sil}
\end{equation}

By applying the chain rule, we then get
\begin{equation}
\frac{\partial}{\partial F(x_{ij})} \exp\{s_{il}(F, \theta)\} =  \exp\{s_{il}(F,\theta)\} \cdot \frac{e_{ij\ell}} {\left(\sigma_{ij\ell}\right)^2}.
\label{eq:derivative_exp_sil}
\end{equation}

By applying the chain rule, Equation~\eqref{eq:derivative_sil} and Equation~\eqref{eq:derivative_exp_sil}, we obtain
\begin{equation}
\begin{aligned}
    - \frac{\partial}{\partial F(x_{ij})} \log\left[
        \frac{1}{L}
        \sum_{\ell=1}^{L}
        \exp\!\left\{
            s_{i\ell}(F,\theta)
        \right\}
    \right] 
    &= - \frac{
        \frac{1}{L}
        \sum_{\ell=1}^{L}
        \exp\!\left\{
            s_{i\ell}(F,\theta)
        \right\} \cdot\frac{e_{ij\ell}} {\left(\sigma_{ij\ell}\right)^2}}{
        \frac{1}{L}
        \sum_{\ell=1}^{L}
        \exp\!\left\{
            s_{i\ell}(F,\theta) 
        \right\}}\\
    &= - \sum_{\ell=1}^{L} w_{i\ell} \frac{e_{ij\ell}} {\left(\sigma_{ij\ell}\right)^2}.
\end{aligned}
\end{equation}

Adding the indices again, gives the result from Equation~\eqref{eq:boosting_gradient}
\[
g_{ij}^{(t,b)} = -
    \sum_{\ell=1}^{L}
    w_{i\ell}^{(t, b)}
    \frac{e_{ij\ell}^{(t, b)}}{\left(\sigma_{ij\ell}^{(t, b)}\right)^2}.
\]
\hfill 

To derive the Hessian, we suppress the fixed indices \(t\), and \(b\) again.

The Hessian is given by
\[
h_{ij}=\frac{\partial^2 \widehat{\mathcal R}(F,\theta,\phi)}{\partial F(x_{ij})^2} = \frac{\partial}{\partial F(x_{ij})} \left[ - \sum_{\ell=1}^{L} w_{i\ell} \frac{e_{ij\ell}} {\left(\sigma_{ij\ell}\right)^2} \right].
\]

By applying the product rule, and using
\begin{equation} 
\begin{aligned} 
\frac{\partial w_{i\ell}}{\partial F(x_{ij})} = w_{i\ell} \Bigg[ \frac{e_{ij\ell}} {\left(\sigma_{ij\ell}\right)^2} - \sum_{q=1}^{L} w_{iq} \frac{e_{ijq}} {\left(\sigma_{ijq}\right)^2} \Bigg], 
\end{aligned} 
\label{eq:derivative_wil}
 \end{equation}
and 
\begin{equation} 
\begin{aligned} 
\frac{\partial e_{ij\ell}}{\partial F(x_{ij})} = -1, 
\end{aligned} 
\label{eq:derivative_eijl}
 \end{equation}

we have
\begin{equation}
\begin{aligned}
\frac{\partial}{\partial F(x_{ij})} \left[- \sum_{\ell=1}^{L} w_{i\ell} \frac{e_{ij\ell}} {\left(\sigma_{ij\ell}\right)^2} \right]
    &= - \sum_{\ell=1}^{L} \left( w_{i\ell} \Bigg[ \frac{e_{ij\ell}} {\left(\sigma_{ij\ell}\right)^2} - \sum_{q=1}^{L} w_{iq} \frac{e_{ijq}} {\left(\sigma_{ijq}\right)^2} \Bigg] \frac{e_{ij\ell}} {\left(\sigma_{ij\ell}\right)^2} + w_{i\ell}\frac{-1}{\sigma_{ij\ell}^2} \right)\\
    &= \sum_{\ell=1}^{L} w_{i\ell} \frac{1} {\left(\sigma_{ij\ell}\right)^2} - \sum_{\ell=1}^{L} w_{i\ell} \left( \frac{e_{ij\ell}} {\left(\sigma_{ij\ell}\right)^2} \right)^2  \left[ \sum_{\ell=1}^{L} w_{i\ell} \frac{e_{ij\ell}} {\left(\sigma_{ij\ell}\right)^2} \right]^2.
\end{aligned}
\end{equation}

Adding the indices again, gives
\begin{equation} 
\begin{aligned} 
 h_{ij}^{(t,b)} ={}
 & \sum_{\ell=1}^{L}w_{i\ell}^{(t,b)} \frac{1} {\left(\sigma_{ij\ell}^{(t,b)}\right)^2} - \sum_{\ell=1}^{L} w_{i\ell}^{(t,b)} \left( \frac{e_{ij\ell}^{(t,b)}} {\left(\sigma_{ij\ell}^{(t,b)}\right)^2} \right)^2 + \left[ \sum_{\ell=1}^{L} w_{i\ell}^{(t,b)} \frac{e_{ij\ell}^{(t,b)}} {\left(\sigma_{ij\ell}^{(t,b)}\right)^2} \right]^2. 
\end{aligned} 
\label{eq:boosting_hessian_derivation} 
\end{equation} 
\hfill 

The entries of the diagonal Hessian are not guaranteed to be non-negative, and our implementation currently does not clip negative entries. In our experiments, negative values occurred very rarely and did not lead to numerical issues. However, negative Hessians could easily be clipped to a small positive value if required. The gradient and Hessian expressions apply to every \(j\in\mathcal T_i^{(b)}\). The split-specific contributions are assembled across context--target splits according to Section~\ref{ss:context_target_gradient}, producing the gradient and diagonal-Hessian vectors. 
\section{Algorithm}
\label{app:npboost_training}
\begin{algorithm}[h!]
\caption{NPBoost training}
\begin{algorithmic}[1]

\REQUIRE Training tasks $\{(X_i,y_i)\}_{i=1}^K$;
initial fixed-effects predictor $F^{(0)}$;
initial NP parameters $(\theta^{(0)}, \phi^{(0)})$;
NP epochs per round $E$;
number of boosting rounds $T$;
boosting learning rate $\nu$

\ENSURE Trained model
$(F^{(T)},\theta^{(T)}, \phi^{(T)})$

\FOR{$t=1,\dots,T$}

    \STATE Compute the residualized responses
    \[
        \widetilde y_i^{(t-1)}
        =
        y_i-F^{(t-1)}(X_i),
        \qquad i=1,\dots,K.
    \]

    \STATE Update
    $(\theta^{(t-1)}, \phi^{(t-1)})$ to
    $(\theta^{(t)}, \phi^{(t)})$
    by $E$ epochs of stochastic minimization of the Monte Carlo predictive risk in
    Equation~\eqref{eq:npboost_empirical_risk}, using 
    $\{(X_i, \widetilde y_i^{(t-1)})\}_{i=1}^K$
    and randomly sampled context--target splits.

    \STATE Construct the context--target splits used for the tree update
    according to Section~\ref{ss:context_target_gradient}.

    \STATE Evaluate the split-specific first- and second-order derivatives
    using
    $(F^{(t-1)},\theta^{(t)}, \phi^{(t)})$,
    and aggregate them to obtain the complete gradient vector
    $g^{(t)}$ and diagonal-Hessian vector $h^{(t)}$.

    \STATE Fit the regression tree
    \[
        f^{(t)}
        \in
        \argmin_{f\in\mathcal H}
        \sum_{i=1}^{K}
        \sum_{j=1}^{N_i}
        h_{ij}^{(t)}
        \left(
            -\frac{g_{ij}^{(t)}}{h_{ij}^{(t)}}
            -
            f(x_{ij})
        \right)^2.
    \]

    \STATE Update the fixed-effects predictor:
    \[
        F^{(t)}(\cdot)
        =
        F^{(t-1)}(\cdot)
        +
        \nu f^{(t)}(\cdot).
    \]

\ENDFOR

\RETURN $(F^{(T)},\theta^{(T)}, \phi^{(T)})$

\end{algorithmic}
\end{algorithm}

\section{Test-Time Computational Complexity}
\label{ss:prediction_complexity}

Let $n_c=|\mathcal C_i|$ and $n_t=|\mathcal T_i|$, and let $L$
denote the number of latent samples.
A standard latent NP has complexity
$
    \mathcal O\!\left(n_c + L n_t\right).
$
For fixed $L$, this is linear in the combined number of context and target
observations. Let $M$ denote the number of trees in the fixed-effects ensemble and $D$
their average depth.
NPBoost additionally evaluates the tree ensemble at the context and target
inputs, giving $
    \mathcal O\!\left(
        n_c + L n_t + MD(n_c+n_t)
    \right)
$ complexity.
Consequently, for fixed $L$, $M$, and $D$, NPBoost retains linear prediction
complexity in $n_c+n_t$.

For Attentive-NPs and ANPBoost, the deterministic cross-attention path compares every target query with every context key and therefore adds a term of order $\mathcal O(n_c n_t)$. This is quadratic when $n_c$ and $n_t$ grow proportionally, but remains linear in $n_t$ in the few-shot setting when $n_c$ is fixed. Our implementation does not use context self-attention, which would add an additional $\mathcal O(n_c^2)$ term.

\
\section{Robustness Experiments}
\label{app:robust}
We consider four variants of the synthetic experiment of the main text, each designed to examine
a different aspect of the problem.
The \dataset{more-tasks} setting redistributes the same $60{,}000$ observations
across $1{,}200$ tasks with $50$ observations each, allowing us to study the
effect of observing more task realizations and fewer
observations per task.
In this setting, the validation and test tasks only provide $10$ labeled context points.
The \dataset{Brownian} setting replaces the RBF process with Brownian paths,
while the \dataset{Brownian-LN} setting additionally replaces the
Gaussian noise with centered lognormal noise of mean zero and variance
$0.25$. Note that Brownian paths do not have constant marginal variance across the input domain, since the marginal variance grows with the value of each input feature. Thus the random-effects component is no longer normalized to have variance one.
Finally, the \dataset{irr-feat} setting adds two irrelevant features to the original two-dimensional input. 
The original features influence both the shared and task-specific components, whereas the two additional features have no effect.
Complete specifications and representative samples for fixed-effects, random-effects and noise components are provided in
Appendix~\ref{app:synthetic_data_generation}.

\section{Synthetic Datasets}
\label{app:synthetic_data_generation}

Here we provide details for the synthetic data-generating process used throughout Section~\ref{s:experiments}.
We describe each component (fixed effects, task-specific random effects, and observation noise) in the following subsections.

For task \(i\) and observation \(j\), responses are generated according to
\begin{equation}
    y_{ij}
    =
    F(x_{ij})
    +
    b_i(x_{ij})
    +
    \varepsilon_{ij},
    \label{eq:app_synthetic_model}
\end{equation}
where \(F\) is shared across tasks, \(b_i\) is a task-specific stochastic function, and \(\varepsilon_{ij}\) denotes observation noise.
Unless stated otherwise, inputs are sampled independently \(x_{ij} \sim \operatorname{Unif}([-2,2]^d)\), and each dataset contains \(600\) tasks with \(100\) observations per task.

\subsection{Fixed-Effects Function}

We consider a piecewise constant fixed-effects function \(F\) in one or two input dimensions, referred to as Steps~1D and Steps~2D, respectively.
The Steps~1D function is built from the unnormalized piecewise-constant function 
\begin{equation}
    \bar F(x) = 
        3 \cdot \mathbf{1}_{\{-2 \leq x < -1\}}
        -2 \cdot \mathbf{1}_{\{-1 \leq x < 0\}}
        +5 \cdot \mathbf{1}_{\{0 \leq x < 0.5\}}
        - 1 \cdot \mathbf{1}_{\{0.5 \leq x \leq 2\}}.
    \label{eq:app_steps_1d}
\end{equation}
We rescale \(\bar F\) to unit variance by setting \(F = c \cdot \bar F\), where \[c = \sqrt{\frac{1}{ \operatorname{Var}(\bar F(x))}}.\]
The two-dimensional function is constructed analogously as a piecewise-constant surface with variance \(1\).
We omit its full definition because it requires a cumbersome enumeration of rectangular regions.
Figure~\ref{fig:app_steps} visualizes both functions.

\begin{figure}[!htbp]
    \centering
    \begin{subfigure}[b]{0.48\linewidth}
        \centering
        \includegraphics[width=\linewidth]{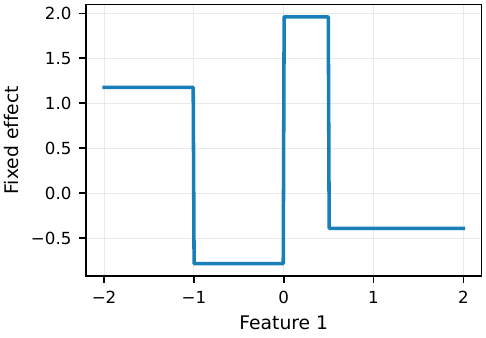}
        \caption{Steps \(1\)D.}
        \label{fig:app_steps_1d}
    \end{subfigure}
    \hfill
    \begin{subfigure}[b]{0.48\linewidth}
        \centering
        \includegraphics[width=\linewidth]{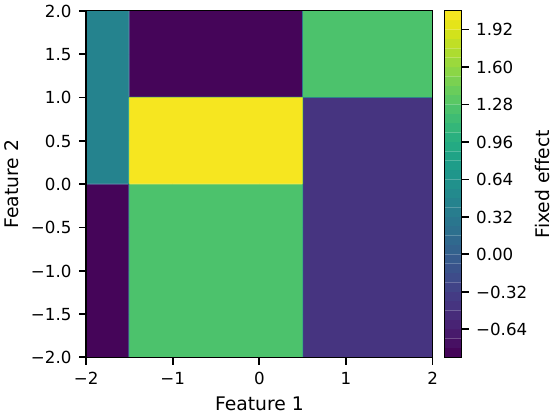}
        \caption{Steps \(2\)D.}
        \label{fig:app_steps_2d}
    \end{subfigure}
    \caption{\textbf{Fixed-effects functions.} Piecewise-constant functions used to generate the shared structure across tasks.}
    \label{fig:app_steps}
\end{figure}

\paragraph{Normalization constant for Steps~1D.}

From \(\mathbb{E}[\bar F(x)] = \frac{1}{2}\) and \(\mathbb{E}[\bar F(x)^2] = \frac{27}{4}\), it follows that \(\operatorname{Var}(\bar F(x)) =  \frac{27}{4} - \left(\frac{1}{2} \right)^2 = \frac{13}{2}\)
and therefore \(c=\sqrt{\frac{2}{13}}.\)

\subsection{Task-Specific Processes}

In the reference setting, the task-specific functions are sampled independently from a zero-mean Gaussian process with RBF kernel,
\begin{equation}
    b_i \sim \mathcal{GP}(0,k),
    \qquad
    k(x,x')
    =
    \exp\left\{
        -\frac{\lVert x-x'\rVert_2^2}{2\cdot 0.5^2}
    \right\}.
    \label{eq:app_rbf_kernel}
\end{equation}

In the Brownian settings, the RBF process is replaced by a (shifted) Brownian-sheet kernel \citep{khoshnevisan2002multiparameter}. Writing
\(x=(x_1,\ldots,x_d)\) and \(x'=(x'_1,\ldots,x'_d)\), we use
\begin{equation}
    k(x,x')
    =
    \sigma_b^2
    \prod_{\ell=1}^{d}
    \min\left(x_{\ell}-a,\,x'_{\ell}-a\right),
    \qquad
    \sigma_b^2=1,
    \quad
    a=-2.
    \label{eq:app_brownian_kernel}
\end{equation}
For \(d=1\), this is the covariance kernel of Brownian motion starting at \(a=-2\).
Unlike the RBF process, this process is non-stationary and produces rough sample paths.
Figure~\ref{fig:app_random} shows representative one-dimensional realizations.

\begin{figure}[!htbp]
    \centering
    \begin{subfigure}[b]{\linewidth}
        \centering
        \includegraphics[width=\linewidth]{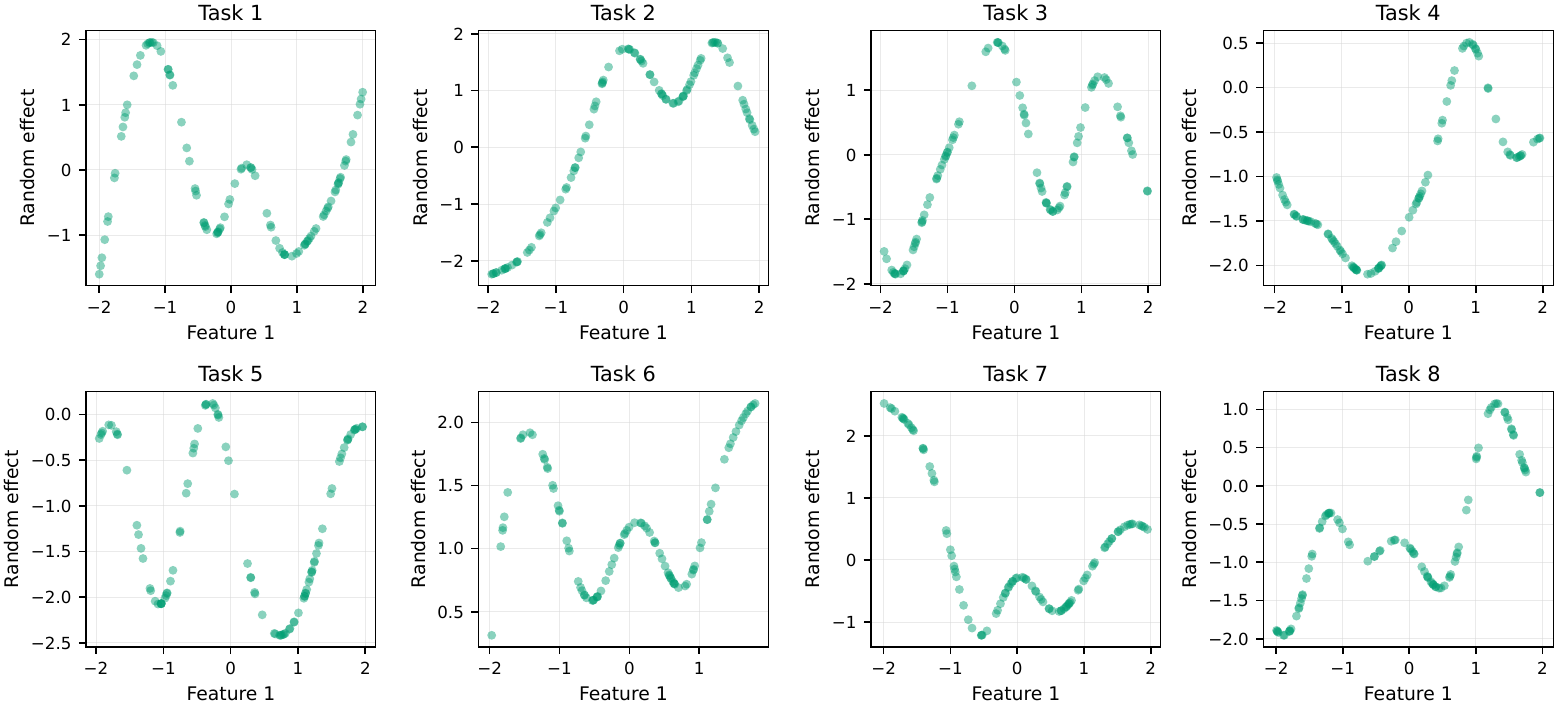}
        \caption{RBF Gaussian process.}
        \label{fig:app_random_gp}
    \end{subfigure}
    \\[1.5ex]
    \begin{subfigure}[b]{\linewidth}
        \centering
        \includegraphics[width=\linewidth]{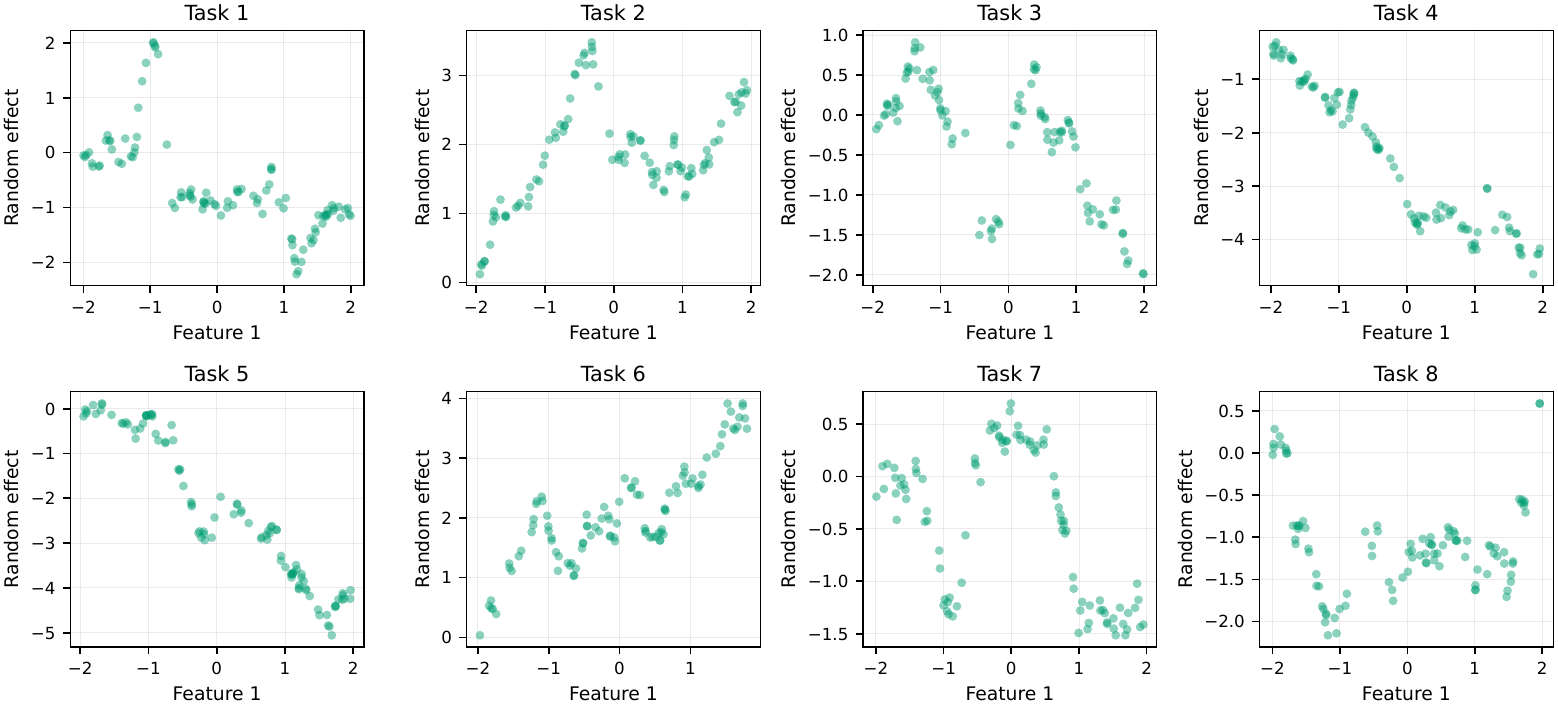}
        \caption{Brownian motion.}
        \label{fig:app_random_brownian}
    \end{subfigure}
    \caption{\textbf{Task-specific processes.} Independent realizations of eight task-specific processes.}
    \label{fig:app_random}
\end{figure}

\subsection{Observation Noise}

The default observation noise is sampled independently as
\begin{equation}
    \varepsilon_{ij}\sim\mathcal{N}(0,0.25).
    \label{eq:app_gaussian_noise}
\end{equation}

The \dataset{Brownian-LN} setting instead uses centered lognormal noise.
Let
\begin{equation}
    Z \sim \operatorname{LogNormal}(\mu,\sigma^2),
    \qquad
    \sigma=1,
\end{equation}
where \(\mu\) and \(\sigma\) denote the mean and standard deviation on the log scale.
To match the variance in Equation~\eqref{eq:app_gaussian_noise}, we set
\begin{equation}
    \mu
    =
    \frac{1}{2}
    \log\left(
        \frac{0.25}{\exp\{\sigma^2\}-1}
    \right)
    -
    \frac{\sigma^2}{2},
\end{equation}
and define
\begin{equation}
    \varepsilon_{ij}
    =
    Z-\exp\left\{\mu+\frac{\sigma^2}{2}\right\}.
\end{equation}
Consequently, \(\mathbb{E}[\varepsilon]=0\) and \(\operatorname{Var}(\varepsilon)=0.25\), while the noise distribution remains asymmetric.

\subsection{GP Regression}
\label{app:gp_regression}
In addition to the configurations described in Section~\ref{ss:synthetic_datasets}, we also examine the \dataset{GP-Regression} setting, in which we use a 1D or 2D GP with an RBF kernel and Gaussian noise, but with \(F(\cdot ) \equiv 0\).
We include this experiment because it aligns with the setup commonly used in many NP benchmarks (e.g. \citet{kim2019iclr-attentive, mortimer2026incremental}).
It also acts as a sanity check to ensure that our proposed method does not unexpectedly perform poorly when the fixed effects are zero.
We present its results together with all synthetic settings from the main text in Appendix~\ref{app:complete_results}.

\subsection{Data Splits and Context-Set Construction}
\label{app:synthetic_splits}

For each random seed, we generate a new dataset.
For within-task prediction, all tasks are used during training.
Within each task, observations are randomly divided into \(50\%\) training, \(25\%\) validation, and \(25\%\) test observations.
The training observations provide the task-specific context, and performance is evaluated on the corresponding held-out observations.

For few-shot prediction, entire tasks are partitioned into training, validation, and test sets.
In settings with \(600\) tasks, we use \(400\) training tasks, \(100\) validation tasks, and \(100\) test tasks.
Each validation or test task provides \(20\) randomly selected labeled context points, and predictions are evaluated on its remaining \(80\) observations.
In the \emph{more-tasks} setting, we use \(800\) training tasks, \(200\) validation tasks, and \(200\) test tasks, each with only \(50\) observations.
In the few-shot setting, validation and test tasks provide \(10\) labeled context points, and predictions are evaluated on the remaining \(40\) observations.

Figures~\ref{fig:reference_within_task} and \ref{fig:reference_few_shot} visualize samples from tasks in the \dataset{reference} setting. 
The raw observations are identical for the within-task and the few-shot setting. 
Only the train-validation-test split changes. 
In within-task, each task is split internally into train, validation, and test points.
In few-shot, entire tasks are used for training, whereas some tasks are held out for validation and some for testing.
The validation and test tasks are further split into a small context set and a target set on which predictions are evaluated.

\begin{figure}[!htbp]
    \centering
    \includegraphics[width=\linewidth]{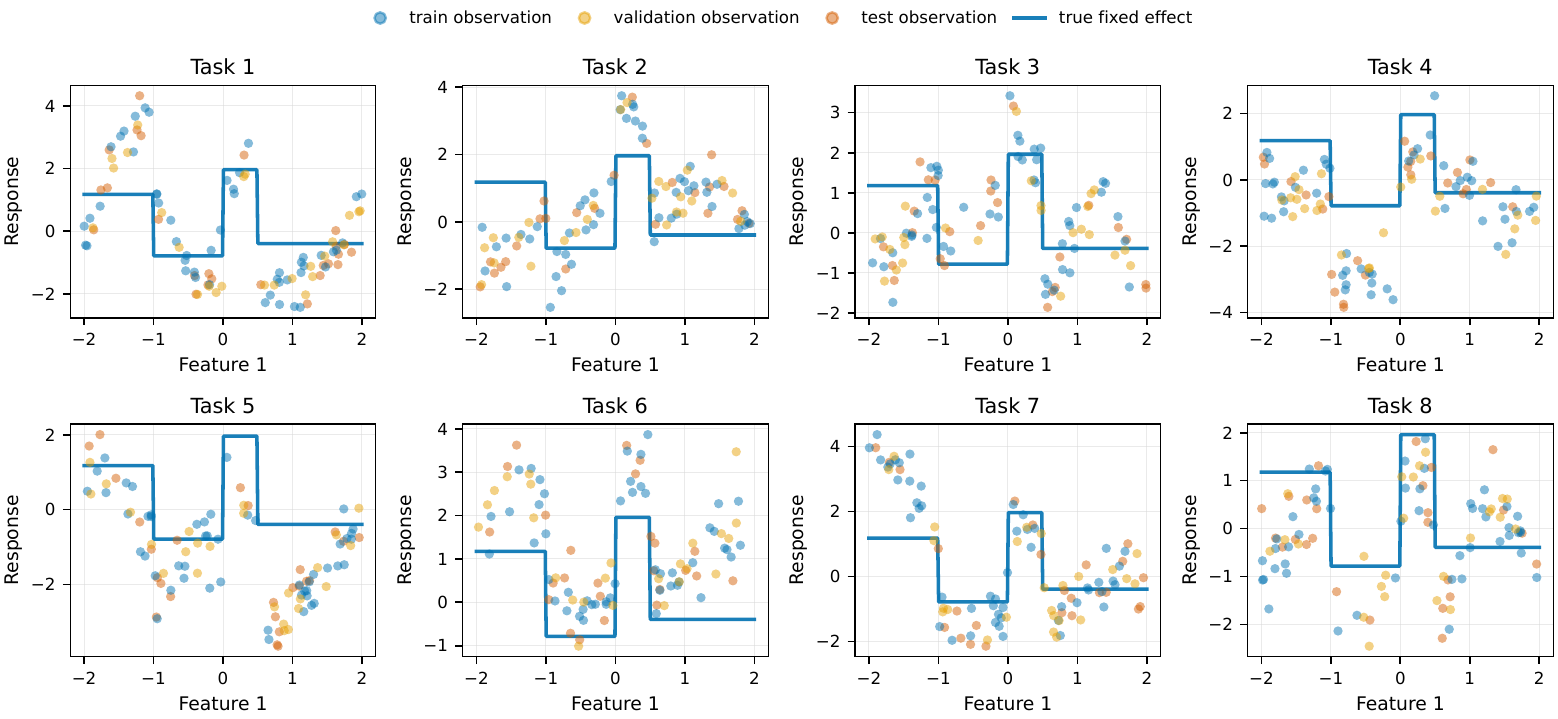}
    \caption{Representative tasks from the \dataset{reference} experiment with Steps~1D fixed effects for the within-task split.}
    \label{fig:reference_within_task}
\end{figure}

\begin{figure}[!htbp]
    \centering
    \includegraphics[width=\linewidth]{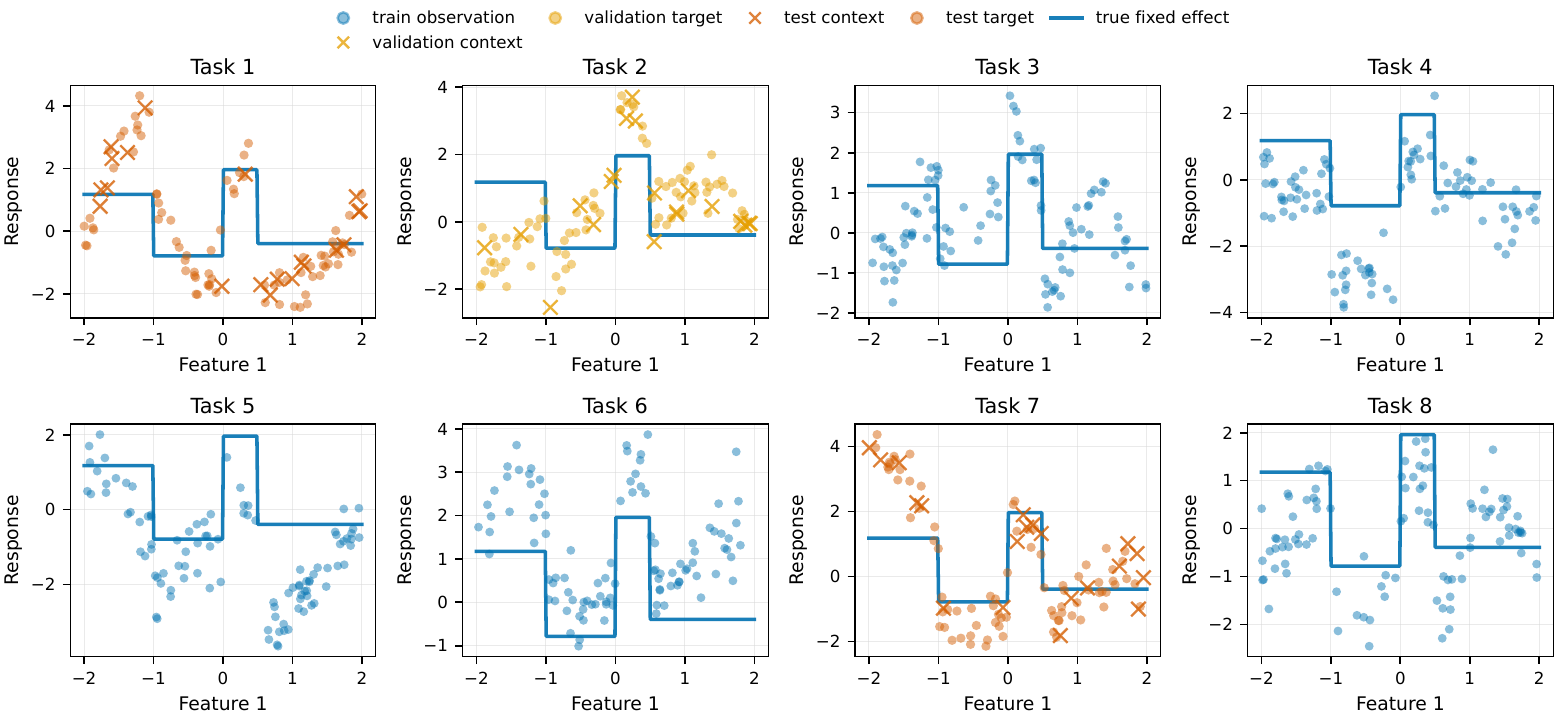}
    \caption{Representative tasks from the \dataset{reference} experiment with Steps~1D fixed effects for the few-shot split.}
    \label{fig:reference_few_shot}
\end{figure}

\section{Evaluation Criteria}
\label{app:evaluation_criteria}

For all prediction scenarios, we use the root mean square error (RMSE) to
measure point prediction accuracy of the mean of the predictive
distribution. For the test tasks $i\in\mathcal I_{\mathrm{test}}$, with
$N_{\mathrm{test}}
=
\sum_{i\in\mathcal I_{\mathrm{test}}}|\mathcal T_i|$
test observations in total, the RMSE is given by
\begin{equation}
    \operatorname{RMSE}
    =
    \sqrt{
        \frac{1}{N_{\mathrm{test}}}
        \sum_{i\in\mathcal I_{\mathrm{test}}}
        \left\|
            y_i^{\mathcal T_i}
            -
            \widehat\mu_i^{\mathcal T_i}
        \right\|_2^2
    },
    \label{eq:rmse}
\end{equation}
where $\widehat\mu_i^{\mathcal T_i}$ is the mean of the predictive
distribution at the target inputs $X_i^{\mathcal T_i}$.

Furthermore, we use the continuous ranked probability score (CRPS) for
models that provide probabilistic predictions. The CRPS is a strictly proper
scoring rule \citep{gneiting2007strictly} and is given by
\begin{equation}
    \operatorname{CRPS}
    =
    \frac{1}{N_{\mathrm{test}}}
    \sum_{i\in\mathcal I_{\mathrm{test}}}
    \sum_{j\in\mathcal T_i}
    \int
    (
        \widehat{\operatorname{CDF}}_{ij}(u)
        -
        \mathbbm{1}\{y_{ij}\leq u\}
    )^2
    \,\mathrm{d}u,
    \label{eq:crps}
\end{equation}
where $\widehat{\operatorname{CDF}}_{ij}(\cdot)$ is the predictive
cumulative distribution function for observation $j$ in task $i$.
As we do not have a closed-form expression for the predictive cumulative
distribution function of the NP-based models, we approximate
the CRPS with a Monte Carlo estimate using the computationally
efficient implementation in \citet{bingham2019pyro}.

\section{Real-World Datasets and Preprocessing}
\label{app:real_world_datasets}

\begin{table}[h!]
\caption{Real-world datasets after preprocessing.
$K$ and $N$ denote the number of tasks and observations, respectively.}
\label{tab:real_world_datasets}
\centering
\setlength{\tabcolsep}{1.75pt}
\begin{tabular}{@{}l p{0.28\columnwidth} p{0.28\columnwidth} cc@{}}
\toprule
Dataset & Response & Task & $K$ & $N$ \\
\midrule
\texttt{cars}
    & \(\log(\text{price})\)
    & car model
    & 1,255
    & 72,518 \\
\texttt{Spotify}
    & danceability
    & artist
    & 418
    & 7,868 \\
\texttt{cows}
    & milk protein
    & cow
    & 79
    & 1,337 \\
\texttt{bikes}
    & $\log(1+\text{count})$
    & calendar date
    & 353
    & 8465 \\
\bottomrule
\end{tabular}
\end{table}

\subsection{Preprocessing and Splits}
\label{app:real_preprocessing_and_splits}

We remove tasks with fewer than 10 observations.
For each seed, we construct the train--validation--test partitions and, in the few-shot scenario, the context--target partitions for the validation and test tasks before applying any preprocessing transformations.
The same partitions are used for all methods.

Continuous predictors are standardized using the mean and standard deviation estimated from the corresponding training data.
Categorical predictors are one-hot encoded using categories observed in the training data.

For within-task prediction, the observations of each task are randomly divided into 50\% training, 25\% validation, and 25\% test observations.
For few-shot prediction, tasks are partitioned into disjoint training, validation, and test sets.
Each validation or test task provides \(7\) labeled context points, and predictions are evaluated on all remaining observations.

For models with fixed and task-specific components, all predictors enter the fixed component, whereas only continuous predictors are provided for the task-specific component.

\subsection{Cars}
\label{sss:cars}
The dataset \dataset{cars} is available on Kaggle \citep{reese2021cars} and lists prices of used cars in the United States.
We follow the preprocessing procedure of \citet{simchoni2023integrating}, with minor modifications.
We predict log-transformed prices, \(\log(\text{price})\).
Each task corresponds to a car model. 
We use the continuous features latitude, longitude, odometer, and year, imputing NAs with the training-set median. 
We use manufacturer, drive, transmission, and condition as categorical predictors, assigning missing values to an NA-category.
We remove car models with fewer than \(10\) observations, which leaves \(1{,}255\) tasks and \(72{,}518\) observations.
For few-shot, we hold out \(200\) tasks for validation and \(200\) tasks for testing.
RMSE and CRPS are reported on the log scale.

\subsection{Spotify}
\label{app:spotify}
The dataset \dataset{Spotify} was published through TidyTuesday \citep{dslc2024tidytuesday}. 
It contains audio characteristics of tracks collected from Spotify.
We predict danceability and define tasks by artist.
The continuous predictors are energy, loudness, speechiness, acousticness, instrumentalness, liveness, valence, tempo, and duration.
Playlist genre, musical key, and mode are treated as categorical predictors.
We remove artists with fewer than \(10\) observations, which leaves \(418\) tasks and \(7{,}868\) observations.
For few-shot prediction, we hold out \(100\) tasks for validation and \(100\) tasks for testing.

\subsection{Cows}
\label{app:cows}
The dataset \dataset{cows} is the Milk dataset from the R package \texttt{nlme} \
\citep{pinheiro2026nlme}, and originally appeared in \citet{diggle2002analysis}.
It contains repeated measurements of milk protein content following calving.
We predict milk protein content and define tasks by cow.
The features are the time since calving (continuous) and diet (categorical).
With \(1{,}337\) observations from 79 cows, it is a small but standard
mixed-effects benchmark.
For few-shot prediction, we hold out \(15\) tasks for validation and \(15\) tasks for testing.

\subsection{Bikes}
\label{app:bikes}
The \dataset{bikes} dataset is the Seoul Bike Sharing Demand dataset from the UCI Machine Learning Repository \citep{uci2020seoulbike}.
We predict log-transformed hourly bicycle-rental counts, \(\log(1+\text{count})\), and define tasks by calendar date.
To represent temporal periodicity, we use cyclic encoding for hour of day, day of week, and day of year \citep{adams1998encoding}.
This encoding preserves the distance across cycle boundaries, such as Sunday and Monday or the end and beginning of a calendar year.
We remove observations marked as non-functioning hours.
The processed dataset contains \(352\) tasks with \(24\) observations each and one task with \(17\) observations.
The splits for within-task prediction are done as usual, whereas for few-shot prediction, we use temporally ordered task partitions shown in Table~\ref{tab:temporal_splits}.
RMSE and CRPS are reported on the log scale.

\begin{table}[!htbp]
    \centering
    \caption{Temporal train-validation-test splits for few-shot prediction for the \dataset{bikes} dataset. Validation and test tasks are randomly sampled from the indicated period following the training window.}
    \label{tab:temporal_splits}
    \begin{tabular}{ccc}
        \toprule
        Seed & Training tasks & Validation and test tasks randomly sampled from \\
        \midrule
        \(0\) & First \(2\) months & Following \(2\) months \\
        \(1\) & First \(3\) months & Following \(2\) months \\
        \(2\) & First \(4\) months & Following \(4\) months \\
        \(3\) & First \(6\) months & Following \(4\) months \\
        \(4\) & First \(8\) months & Following \(4\) months \\
        \bottomrule
    \end{tabular}
\end{table}

\section{Implementation Details}
\label{app:implementation_details}

All experiments were implemented in Python 3.12.13. All NP components were implemented in PyTorch 2.8.0 following the implementation of \citet{dubois2020npf}. The boosted-tree components of NPBoost and ANPBoost, as well as the GBT and Task-ID GBT baselines, used LightGBM 4.6.0. 
GPLinear and LME were fitted using GPBoost 1.6.1 \citep{sigrist2022gpboost}, and the TabICL baselines were run using TabICL 2.1.1.

We perform parameter tuning and early stopping based on validation RMSE.
For details on the tuning, we refer to Appendix~\ref{app:parameter_tuning}. 
For the boosted trees, (A)NP and (A)NPBoost models, we tune their parameters using random search \citep{bergstra2012random} with $32$ trials.
For (A)NPBoost, we calculate the validation RMSE after every boosting round.
We train for at most $500$ rounds and apply early stopping with a patience of $25$ rounds based on validation RMSE.
We use the same early-stopping patience for the gradient boosted trees.
For the NP models, we train for at most $4000$ epochs and apply early stopping with a patience of $200$ epochs based on validation RMSE. 
For every dataset, prediction scenario and seed, we select the model that achieves the best validation RMSE and report its corresponding test RMSE and test CRPS. 
The final test metrics are obtained by averaging these values across the 5 seeds. 
For the NP-based models, we estimate CRPS using 400 samples from the
predictive distribution.
The $400$ samples are obtained by sampling the latent variable in the NP model $20$ times. 
For each latent sample, the NP decoder outputs a conditional predictive distribution, from which we draw $20$ samples.

\ifarxiv
The source code is available at
\url{https://github.com/ken-r/npboost}.
\fi
The NP components of NPBoost and ANPBoost follow the implementation of  \citet{dubois2020npf}, and their boosted-tree components use LightGBM. 
Appendix~\ref{app:hardware} describes the hardware on which we ran the experiments. 

We fix the following NP architecture.
The representation dimension and the latent dimension are  \(d_r = d_z = 128\); the encoder networks \(h_{\phi_h}\) and \(g_{\phi_g}\) have hidden dimensions equal to \([256,256]\); the decoder network has hidden dimensions equal to \([128,128,128,128]\); all activation functions are ReLU; and the number of latent samples is \(L=20\).
We sample new context--target splits in every epoch, with context proportion equal to \(0.5\) in the within-task scenario and context size equal to the number of support points provided during validation and testing in the few-shot scenario.

For the attentive NP models (Figure~\ref{fig:app_anp_architecture}), the latent path is
augmented by a deterministic path with its own encoder of hidden dimensions \([256, 256]\) and
cross-attention with \(8\) heads.
Cross-attention uses the target inputs $X_i^{\mathcal T_i}$ as queries, the context inputs $X_i^{\mathcal C_i}$ as keys and the deterministic context representations (denoted by $r_i^{\text{det}, \mathcal C_i}$) as values.
This produces a target-specific deterministic representation (denoted by $r_i^{\text{det}, \mathcal T_i}$) which is passed to the decoder.
Figure~\ref{fig:app_np_architecture_overview} summarizes the architecture used for the NP components.

\begin{figure}[!htbp]
    \centering
    \begin{subfigure}[b]{0.48\linewidth}
        \centering
        \includegraphics[width=\linewidth]{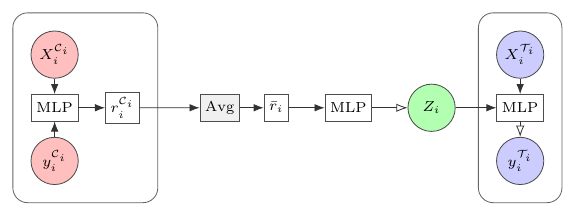}
        \caption{NP.}
        \label{fig:app_np_architecture}
    \end{subfigure}
    \hfill
    \begin{subfigure}[b]{0.48\linewidth}
        \centering
        \includegraphics[width=\linewidth]{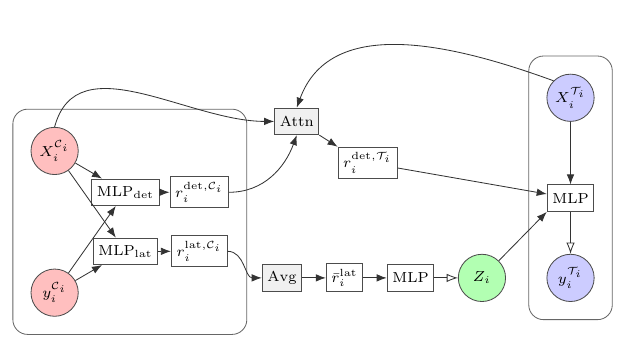}
        \caption{ANP.}
        \label{fig:app_anp_architecture}
    \end{subfigure}
    \caption{\textbf{Neural process architectures.} Following the style of \citet{dubois2020npf}, we show an overview of the NP and ANP architecture used in our implementation. Avg denotes taking the average, and Attn denotes cross-attention.
    Filled arrows represent deterministic information flow, while hollow arrows denote distributional outputs.
    }
    \label{fig:app_np_architecture_overview}
\end{figure}

For NP training, we optimize all model parameters using Adam \citep{kingma2014adam}.
We use a batch size of \(16\) tasks, and \(64\) for the \dataset{more-tasks} experiment.
We do not use padding.
Instead, we group tasks of the same size into batches and use gradient accumulation until at least \(16\) (or \(64\)) tasks have been processed, and then perform an optimizer step.
Gradients are clipped to a norm of \(1\) before each step.

In our decoder implementation, the conditional predictive variance is
$\sigma^2_\theta(x_{ij}, z_i) = \sigma_e^2(1+s_\theta(x_{ij},z_i))$, where $s_\theta(x_{ij},z_i)\geq 0$ is obtained by applying a softplus transformation to the corresponding decoder
output. 
For the additive form presented in Section~\ref{ss:background_on_nps}, we define
$v_\theta(x_{ij},z_i)=\sigma_e^2 s_\theta(x_{ij},z_i)$, such that
$\sigma^2_\theta(x_{ij}, z_i) = \sigma_e^2(1+s_\theta(x_{ij}, z_i)) = v_\theta(x_{ij}, z_i) + \sigma_e^2$, consistent with Equation~\eqref{eq:np_conditional_variance}

\paragraph{Mixed-Effects Baselines}
The mixed-effects baselines GPLinear and LME are fitted with the GPBoost library~\citep{sigrist2022gpboost} under a Gaussian likelihood.
They use linear fixed effects, which include an intercept and all features.
For LME, each task receives its own intercept and its own slope for every continuous feature.
For GPLinear, the task-specific part is a Gaussian process over all continuous features, using a Matérn covariance function with smoothness parameter $\nu = 1.5$.

\paragraph{TabICL Baselines}
Our TabICL baselines are implemented using the second-generation version of TabICL \citep{qu2026tabiclv2}.
TabICL is a pretrained open tabular foundation model and does not require any training or parameter tuning.
We use the default ensemble of \(8\) estimators.
We consider two variants: TabICL\_T (task-wise) and TabICL\_P (pooled).
TabICL\_T is called once per task, and receives only the context points available for that task, so that no task ID has to be passed.
TabICL\_P is called with all training points and, in the few-shot case, additionally with the context points of the task to be predicted.
Consequently, the task ID is necessary, and we pass it as a categorical variable.
In the within-task scenario, one single TabICL call is made, providing all the training points.
Pooled few-shot prediction still requires one call per task, since the context set is formed by all observations of the training tasks and additionally by the context points of the new task.
Point predictions use the predictive mean and are evaluated with the RMSE.
For CRPS, TabICL provides predictive quantiles. Following~\citet{fakoor2023flexible}, we use the fact that the CRPS can be expressed as an integral over quantile levels of the corresponding pinball loss. 
We numerically approximate this integral by discretizing over a grid of \(399\) quantiles.

\section{Parameter Tuning}
\label{app:parameter_tuning}

Table \ref{tab:hyperparameter_search_spaces} lists the hyperparameters that were tuned.
For every dataset, we ran \(32\) trials, using the same random seed for the sampler.
Consequently, across all datasets, prediction settings, and data-generation seeds, we evaluated an identical set of 32 hyperparameter configurations for the same model.

Since NPBoost and ANPBoost have the largest number of hyperparameters, the tuning has to cover a higher-dimensional search space than for the other models.
We therefore use more focused ranges for selected hyperparameters.
The ranges for the principal optimization parameters remain unchanged: the NP and tree learning rates are searched over the same ranges as for the corresponding standalone baselines.
The NP dropout probability is restricted to the endpoints $\{0, 0.2\}$ of the interval used for NP and ANP, and the number of leaves to $\{4, 8, 16\}$.
The maximum number of bins is fixed to  the library default of $255$.

LME and GPLinear have no tuning parameters.
We evaluate the TabICL models using the default configuration, without fine-tuning or hyperparameter optimization.
\begin{table*}[h!] 
\centering
\caption{Hyperparameter search spaces.
All searches use random search with \(32\) trials.
\(\operatorname{Log-Unif}\) distributions sample uniformly on the logarithmic scale.}
\label{tab:hyperparameter_search_spaces}
\small
\begin{tabular}{llll}
\toprule
Model & Component & Hyperparameter & Search space \\
\midrule
NP/ANP & NP & Dropout probability & \(\operatorname{Unif}([0,0.2])\) \\
NP/ANP & NP & Learning rate & \(\operatorname{Log-Unif}([10^{-5},10^{-3}])\) \\
\midrule
NPBoost/ANPBoost & NP & Dropout probability & \(\{0,0.2\}\)\\
NPBoost/ANPBoost & NP & Learning rate & \(\operatorname{Log-Unif}([10^{-5},10^{-3}])\) \\
NPBoost/ANPBoost & Alternation & NP epochs per boosting round & \(\{5,10,20\}\) \\
NPBoost/ANPBoost & Trees & Learning rate & \(\operatorname{Log-Unif}([0.01,1])\) \\
NPBoost/ANPBoost & Trees & \(L_2\) regularization & \(\{0,1\}\) \\
NPBoost/ANPBoost & Trees & Minimum observations per leaf & \(\{10,20,100,1000\}\) \\
NPBoost/ANPBoost & Trees & Number of leaves & \(\{4,8,16\}\) \\
NPBoost/ANPBoost & Trees & Maximum number of bins & \(\{255\}\) \\
\midrule
GBT/Task-ID GBT & Trees & Learning rate & \(\operatorname{Log-Unif}([0.01,1])\) \\
GBT/Task-ID GBT & Trees & \(L_2\) regularization & \(\{0,1\}\) \\
GBT/Task-ID GBT & Trees & Minimum observations per leaf & \(\{10,20,100,1000\}\) \\
GBT/Task-ID GBT & Trees & Number of leaves & \(\{2,4,8,16,32,64\}\) \\
GBT/Task-ID GBT & Trees & Maximum number of bins & \(\{255,500,1000\}\) \\
\bottomrule
\end{tabular}
\end{table*}

\section{Complete Results}
\label{app:complete_results}

Tables~\ref{tab:reference_results}--\ref{tab:irrelevant_features_results} report the complete results for the synthetic dataset, including the \dataset{GP-regression} dataset and the GBT baseline without task ID.
Tables~\ref{tab:cars_results}--\ref{tab:bikes_results} report the complete results for the real-world datasets, including the GBT baseline without task ID.

\begin{table*}[ht!]
    \centering
    \caption{Predictive performance on the synthetic \dataset{reference} dataset. Entries are reported as mean $\pm$ standard error over \(5\) seeds.}
    \label{tab:reference_results}
    \scriptsize
    \setlength{\tabcolsep}{3pt}

    \resizebox{\textwidth}{!}{%
    \begin{tabular}{lllcccccccccc}
        \toprule
        Fixed Effect & Scenario & Metric & NPBoost & ANPBoost & NP & ANP & Task-ID GBT & GBT & GPLinear & LME & TabICL\_T & TabICL\_P \\
        \midrule
        \multirow{4}{*}{Steps $1$D} & within-task & RMSE & \textbf{0.576}$\pm$0.001 & \textbf{0.576}$\pm$0.001 & 0.596$\pm$0.001 & \textbf{0.576}$\pm$0.001 & 0.594$\pm$0.002 & 1.115$\pm$0.004 & 0.685$\pm$0.002 & 1.274$\pm$0.010 & 0.703$\pm$0.003 & 1.104$\pm$0.004 \\
         & within-task & CRPS & \textbf{0.325}$\pm$0.001 & \textbf{0.325}$\pm$0.001 & 0.335$\pm$0.001 & 0.326$\pm$0.001 & -- & -- & 0.379$\pm$0.001 & 0.725$\pm$0.005 & 0.391$\pm$0.001 & 0.621$\pm$0.003 \\
         & few-shot & RMSE & 0.688$\pm$0.010 & 0.659$\pm$0.008 & 0.700$\pm$0.008 & \textbf{0.658}$\pm$0.009 & 1.118$\pm$0.019 & 1.115$\pm$0.019 & 0.836$\pm$0.005 & 1.300$\pm$0.016 & 0.912$\pm$0.007 & 1.137$\pm$0.020 \\
         & few-shot & CRPS & 0.387$\pm$0.006 & 0.371$\pm$0.004 & 0.393$\pm$0.004 & \textbf{0.370}$\pm$0.005 & -- & -- & 0.455$\pm$0.003 & 0.741$\pm$0.009 & 0.508$\pm$0.003 & 0.641$\pm$0.011 \\
        \cmidrule(lr){1-13}
        \multirow{4}{*}{Steps $2$D} & within-task & RMSE & 0.862$\pm$0.003 & \textbf{0.860}$\pm$0.002 & 0.906$\pm$0.002 & 0.866$\pm$0.002 & 0.934$\pm$0.004 & 1.121$\pm$0.003 & 0.953$\pm$0.003 & 1.452$\pm$0.002 & 1.091$\pm$0.005 & 1.125$\pm$0.003 \\
         & within-task & CRPS & 0.488$\pm$0.002 & \textbf{0.484}$\pm$0.001 & 0.519$\pm$0.004 & 0.486$\pm$0.001 & -- & -- & 0.527$\pm$0.001 & 0.824$\pm$0.001 & 0.610$\pm$0.003 & 0.634$\pm$0.002 \\
         & few-shot & RMSE & 1.015$\pm$0.006 & 0.973$\pm$0.006 & 1.031$\pm$0.008 & \textbf{0.972}$\pm$0.006 & 1.118$\pm$0.009 & 1.118$\pm$0.009 & 1.182$\pm$0.002 & 1.456$\pm$0.003 & 1.369$\pm$0.003 & 1.127$\pm$0.008 \\
         & few-shot & CRPS & 0.572$\pm$0.004 & \textbf{0.546}$\pm$0.004 & 0.582$\pm$0.005 & 0.547$\pm$0.004 & -- & -- & 0.653$\pm$0.001 & 0.826$\pm$0.002 & 0.785$\pm$0.002 & 0.635$\pm$0.004 \\
        \bottomrule
    \end{tabular}%
    }
\end{table*}

\begin{table*}[ht!]
    \centering
    \caption{Predictive performance on the synthetic \dataset{GP-regression} dataset without any fixed effects. Entries are reported as mean $\pm$ standard error over \(5\) seeds.}
    \label{tab:gp_regression_results}
    \scriptsize
    \setlength{\tabcolsep}{3pt}

    \resizebox{\textwidth}{!}{%
    \begin{tabular}{lllcccccccccc}
        \toprule
        Fixed Effect & Scenario & Metric & NPBoost & ANPBoost & NP & ANP & Task-ID GBT & GBT & GPLinear & LME & TabICL\_T & TabICL\_P \\
        \midrule
        \multirow{4}{*}{Zero $1$D} & within-task & RMSE & 0.568$\pm$0.002 & 0.567$\pm$0.001 & 0.569$\pm$0.002 & 0.566$\pm$0.002 & 0.579$\pm$0.002 & 1.114$\pm$0.004 & \textbf{0.547}$\pm$0.001 & 0.884$\pm$0.004 & 0.574$\pm$0.001 & 0.776$\pm$0.044 \\
         & within-task & CRPS & 0.322$\pm$0.001 & 0.321$\pm$0.001 & 0.322$\pm$0.001 & 0.320$\pm$0.001 & -- & -- & \textbf{0.309}$\pm$0.001 & 0.507$\pm$0.002 & 0.332$\pm$0.001 & 0.422$\pm$0.025 \\
         & few-shot & RMSE & 0.673$\pm$0.008 & 0.645$\pm$0.006 & 0.689$\pm$0.007 & 0.640$\pm$0.007 & 1.112$\pm$0.019 & 1.112$\pm$0.019 & \textbf{0.607}$\pm$0.004 & 0.908$\pm$0.015 & 0.675$\pm$0.007 & 0.867$\pm$0.049 \\
         & few-shot & CRPS & 0.379$\pm$0.005 & 0.363$\pm$0.004 & 0.388$\pm$0.004 & 0.360$\pm$0.004 & -- & -- & \textbf{0.341}$\pm$0.002 & 0.520$\pm$0.007 & 0.402$\pm$0.003 & 0.474$\pm$0.028 \\
        \cmidrule(lr){1-13}
        \multirow{4}{*}{Zero $2$D} & within-task & RMSE & 0.836$\pm$0.002 & 0.829$\pm$0.003 & 0.843$\pm$0.003 & 0.831$\pm$0.003 & 0.927$\pm$0.003 & 1.119$\pm$0.003 & \textbf{0.749}$\pm$0.002 & 1.049$\pm$0.003 & 0.873$\pm$0.002 & 1.103$\pm$0.010 \\
         & within-task & CRPS & 0.483$\pm$0.003 & 0.466$\pm$0.002 & 0.483$\pm$0.002 & 0.467$\pm$0.002 & -- & -- & \textbf{0.417}$\pm$0.001 & 0.592$\pm$0.002 & 0.497$\pm$0.001 & 0.621$\pm$0.007 \\
         & few-shot & RMSE & 1.011$\pm$0.007 & 0.951$\pm$0.007 & 1.019$\pm$0.008 & 0.950$\pm$0.007 & 1.115$\pm$0.009 & 1.116$\pm$0.009 & \textbf{0.908}$\pm$0.007 & 1.063$\pm$0.007 & 1.016$\pm$0.007 & 1.133$\pm$0.007 \\
         & few-shot & CRPS & 0.570$\pm$0.004 & 0.533$\pm$0.004 & 0.575$\pm$0.005 & 0.533$\pm$0.004 & -- & -- & \textbf{0.505}$\pm$0.004 & 0.601$\pm$0.004 & 0.589$\pm$0.005 & 0.641$\pm$0.004 \\
        \bottomrule
    \end{tabular}%
    }
\end{table*}

\begin{table*}[ht!]
    \centering
    \caption{Predictive performance on the synthetic \dataset{more-tasks} dataset. Entries are reported as mean $\pm$ standard error over \(5\) seeds.}
    \label{tab:more_tasks_results}
    \scriptsize
    \setlength{\tabcolsep}{3pt}

    \resizebox{\textwidth}{!}{%
    \begin{tabular}{lllcccccccccc}
        \toprule
        Fixed Effect & Scenario & Metric & NPBoost & ANPBoost & NP & ANP & Task-ID GBT & GBT & GPLinear & LME & TabICL\_T & TabICL\_P \\
        \midrule
        \multirow{4}{*}{Steps $1$D} & within-task & RMSE & 0.628$\pm$0.003 & \textbf{0.623}$\pm$0.002 & 0.646$\pm$0.002 & \textbf{0.623}$\pm$0.002 & 0.658$\pm$0.003 & 1.119$\pm$0.001 & 0.778$\pm$0.002 & 1.288$\pm$0.003 & 0.833$\pm$0.002 & 1.116$\pm$0.001 \\
         & within-task & CRPS & 0.353$\pm$0.002 & \textbf{0.351}$\pm$0.001 & 0.362$\pm$0.001 & \textbf{0.351}$\pm$0.001 & -- & -- & 0.427$\pm$0.001 & 0.734$\pm$0.001 & 0.468$\pm$0.001 & 0.630$\pm$0.001 \\
         & few-shot & RMSE & 0.767$\pm$0.008 & \textbf{0.743}$\pm$0.008 & 0.773$\pm$0.007 & 0.749$\pm$0.009 & 1.119$\pm$0.010 & 1.119$\pm$0.010 & 1.001$\pm$0.007 & 1.319$\pm$0.008 & 1.121$\pm$0.005 & 1.121$\pm$0.010 \\
         & few-shot & CRPS & 0.427$\pm$0.004 & \textbf{0.414}$\pm$0.004 & 0.431$\pm$0.003 & 0.417$\pm$0.004 & -- & -- & 0.542$\pm$0.004 & 0.752$\pm$0.004 & 0.640$\pm$0.003 & 0.633$\pm$0.006 \\
        \cmidrule(lr){1-13}
        \multirow{4}{*}{Steps $2$D} & within-task & RMSE & 0.974$\pm$0.003 & \textbf{0.954}$\pm$0.004 & 1.001$\pm$0.006 & \textbf{0.954}$\pm$0.004 & 1.032$\pm$0.004 & 1.120$\pm$0.003 & 1.115$\pm$0.004 & 1.462$\pm$0.003 & 1.337$\pm$0.003 & 1.126$\pm$0.003 \\
         & within-task & CRPS & 0.549$\pm$0.002 & 0.537$\pm$0.002 & 0.564$\pm$0.003 & \textbf{0.536}$\pm$0.002 & -- & -- & 0.616$\pm$0.002 & 0.831$\pm$0.001 & 0.764$\pm$0.001 & 0.635$\pm$0.002 \\
         & few-shot & RMSE & 1.043$\pm$0.004 & 1.025$\pm$0.004 & 1.054$\pm$0.005 & \textbf{1.024}$\pm$0.004 & 1.115$\pm$0.005 & 1.115$\pm$0.005 & 1.305$\pm$0.005 & 1.473$\pm$0.009 & 1.511$\pm$0.010 & 1.121$\pm$0.005 \\
         & few-shot & CRPS & 0.588$\pm$0.002 & \textbf{0.577}$\pm$0.003 & 0.595$\pm$0.003 & \textbf{0.577}$\pm$0.002 & -- & -- & 0.728$\pm$0.003 & 0.838$\pm$0.005 & 0.897$\pm$0.006 & 0.633$\pm$0.003 \\
        \bottomrule
    \end{tabular}%
    }
\end{table*}

\begin{table*}[ht!]
    \centering
    \caption{Predictive performance on the synthetic \dataset{Brownian} dataset. Entries are reported as mean $\pm$ standard error over \(5\) seeds.}
    \label{tab:brownian_results}
    \scriptsize
    \setlength{\tabcolsep}{3pt}

    \resizebox{\textwidth}{!}{%
    \begin{tabular}{lllcccccccccc}
        \toprule
        Fixed Effect & Scenario & Metric & NPBoost & ANPBoost & NP & ANP & Task-ID GBT & GBT & GPLinear & LME & TabICL\_T & TabICL\_P \\
        \midrule
        \multirow{4}{*}{Steps $1$D} & within-task & RMSE & 0.610$\pm$0.001 & \textbf{0.608}$\pm$0.002 & 0.628$\pm$0.002 & 0.611$\pm$0.002 & 0.619$\pm$0.002 & 1.485$\pm$0.004 & 0.710$\pm$0.002 & 1.178$\pm$0.005 & 0.724$\pm$0.002 & 1.435$\pm$0.010 \\
         & within-task & CRPS & 0.345$\pm$0.001 & \textbf{0.344}$\pm$0.001 & 0.353$\pm$0.001 & 0.345$\pm$0.001 & -- & -- & 0.393$\pm$0.001 & 0.702$\pm$0.003 & 0.404$\pm$0.001 & 0.773$\pm$0.008 \\
         & few-shot & RMSE & 0.686$\pm$0.004 & \textbf{0.670}$\pm$0.005 & 0.700$\pm$0.005 & 0.673$\pm$0.004 & 1.472$\pm$0.033 & 1.469$\pm$0.034 & 0.874$\pm$0.009 & 1.199$\pm$0.011 & 0.911$\pm$0.006 & 1.535$\pm$0.019 \\
         & few-shot & CRPS & 0.387$\pm$0.002 & \textbf{0.378}$\pm$0.002 & 0.393$\pm$0.002 & 0.380$\pm$0.002 & -- & -- & 0.476$\pm$0.004 & 0.713$\pm$0.005 & 0.514$\pm$0.004 & 0.840$\pm$0.009 \\
        \cmidrule(lr){1-13}
        \multirow{4}{*}{Steps $2$D} & within-task & RMSE & 0.945$\pm$0.004 & \textbf{0.942}$\pm$0.004 & 0.976$\pm$0.004 & 0.952$\pm$0.003 & 0.972$\pm$0.007 & 2.043$\pm$0.008 & 1.077$\pm$0.005 & 1.585$\pm$0.003 & 1.146$\pm$0.003 & 2.020$\pm$0.012 \\
         & within-task & CRPS & 0.525$\pm$0.002 & \textbf{0.521}$\pm$0.002 & 0.541$\pm$0.002 & 0.527$\pm$0.002 & -- & -- & 0.598$\pm$0.003 & 0.930$\pm$0.002 & 0.646$\pm$0.002 & 1.036$\pm$0.004 \\
         & few-shot & RMSE & 1.114$\pm$0.007 & \textbf{1.077}$\pm$0.011 & 1.145$\pm$0.008 & 1.090$\pm$0.013 & 2.004$\pm$0.038 & 2.000$\pm$0.037 & 1.303$\pm$0.008 & 1.630$\pm$0.010 & 1.465$\pm$0.011 & 2.041$\pm$0.029 \\
         & few-shot & CRPS & 0.609$\pm$0.003 & \textbf{0.589}$\pm$0.006 & 0.626$\pm$0.005 & 0.598$\pm$0.007 & -- & -- & 0.718$\pm$0.003 & 0.954$\pm$0.005 & 0.826$\pm$0.005 & 1.053$\pm$0.016 \\
        \bottomrule
    \end{tabular}%
    }
\end{table*}

\begin{table*}[ht!]
    \centering
    \caption{Predictive performance on the synthetic \dataset{Brownian-LN} dataset. Entries are reported as mean $\pm$ standard error over \(5\) seeds.}
    \label{tab:brownian_lognormal_results}
    \scriptsize
    \setlength{\tabcolsep}{3pt}

    \resizebox{\textwidth}{!}{%
    \begin{tabular}{lllcccccccccc}
        \toprule
        Fixed Effect & Scenario & Metric & NPBoost & ANPBoost & NP & ANP & Task-ID GBT & GBT & GPLinear & LME & TabICL\_T & TabICL\_P \\
        \midrule
        \multirow{4}{*}{Steps $1$D} & within-task & RMSE & 0.598$\pm$0.005 & \textbf{0.594}$\pm$0.003 & 0.619$\pm$0.003 & 0.597$\pm$0.004 & 0.622$\pm$0.006 & 1.486$\pm$0.004 & 0.706$\pm$0.006 & 1.179$\pm$0.007 & 0.721$\pm$0.006 & 1.424$\pm$0.014 \\
         & within-task & CRPS & 0.292$\pm$0.001 & \textbf{0.289}$\pm$0.001 & 0.303$\pm$0.001 & 0.291$\pm$0.001 & -- & -- & 0.358$\pm$0.001 & 0.697$\pm$0.003 & 0.335$\pm$0.001 & 0.749$\pm$0.012 \\
         & few-shot & RMSE & 0.658$\pm$0.010 & 0.641$\pm$0.007 & 0.682$\pm$0.012 & \textbf{0.638}$\pm$0.006 & 1.468$\pm$0.037 & 1.465$\pm$0.036 & 0.865$\pm$0.012 & 1.192$\pm$0.008 & 0.902$\pm$0.006 & 1.528$\pm$0.026 \\
         & few-shot & CRPS & 0.343$\pm$0.004 & 0.329$\pm$0.003 & 0.353$\pm$0.006 & \textbf{0.328}$\pm$0.003 & -- & -- & 0.445$\pm$0.004 & 0.706$\pm$0.004 & 0.455$\pm$0.002 & 0.826$\pm$0.015 \\
        \cmidrule(lr){1-13}
        \multirow{4}{*}{Steps $2$D} & within-task & RMSE & \textbf{0.936}$\pm$0.004 & 0.937$\pm$0.003 & 0.974$\pm$0.001 & 0.951$\pm$0.005 & 0.970$\pm$0.007 & 2.044$\pm$0.007 & 1.073$\pm$0.005 & 1.585$\pm$0.005 & 1.138$\pm$0.006 & 2.021$\pm$0.011 \\
         & within-task & CRPS & 0.499$\pm$0.002 & \textbf{0.498}$\pm$0.001 & 0.519$\pm$0.001 & 0.505$\pm$0.002 & -- & -- & 0.582$\pm$0.002 & 0.926$\pm$0.002 & 0.616$\pm$0.001 & 1.026$\pm$0.004 \\
         & few-shot & RMSE & 1.090$\pm$0.010 & \textbf{1.059}$\pm$0.010 & 1.138$\pm$0.012 & 1.068$\pm$0.011 & 2.004$\pm$0.038 & 1.998$\pm$0.039 & 1.297$\pm$0.010 & 1.625$\pm$0.008 & 1.457$\pm$0.011 & 2.041$\pm$0.030 \\
         & few-shot & CRPS & 0.584$\pm$0.005 & \textbf{0.564}$\pm$0.005 & 0.608$\pm$0.006 & 0.570$\pm$0.006 & -- & -- & 0.704$\pm$0.004 & 0.949$\pm$0.005 & 0.811$\pm$0.006 & 1.043$\pm$0.017 \\
        \bottomrule
    \end{tabular}%
    }
\end{table*}

\begin{table*}[ht!]
\centering
    \caption{Predictive performance on the synthetic \dataset{irr-feat} dataset. Entries are reported as mean $\pm$ standard error over \(5\) seeds.}
    \label{tab:irrelevant_features_results}
    \scriptsize
    \setlength{\tabcolsep}{3pt}

    \resizebox{\textwidth}{!}{%
    \begin{tabular}{lllcccccccccc}
        \toprule
        Fixed Effect & Scenario & Metric & NPBoost & ANPBoost & NP & ANP & Task-ID GBT & GBT & GPLinear & LME & TabICL\_T & TabICL\_P \\
        \midrule
        \multirow{4}{*}{Steps $2$D} & within-task & RMSE & 0.917$\pm$0.003 & \textbf{0.884}$\pm$0.002 & 0.951$\pm$0.005 & 0.892$\pm$0.003 & 0.983$\pm$0.004 & 1.121$\pm$0.004 & 1.318$\pm$0.002 & 1.443$\pm$0.004 & 1.259$\pm$0.003 & 1.124$\pm$0.004 \\
         & within-task & CRPS & 0.518$\pm$0.002 & \textbf{0.498}$\pm$0.001 & 0.537$\pm$0.003 & 0.502$\pm$0.002 & -- & -- & 0.742$\pm$0.001 & 0.820$\pm$0.002 & 0.708$\pm$0.001 & 0.634$\pm$0.003 \\
         & few-shot & RMSE & 1.024$\pm$0.006 & \textbf{0.963}$\pm$0.004 & 1.032$\pm$0.005 & 0.968$\pm$0.005 & 1.114$\pm$0.009 & 1.115$\pm$0.008 & 1.403$\pm$0.005 & 1.458$\pm$0.006 & 1.518$\pm$0.004 & 1.120$\pm$0.008 \\
         & few-shot & CRPS & 0.578$\pm$0.003 & \textbf{0.541}$\pm$0.003 & 0.583$\pm$0.003 & 0.544$\pm$0.004 & -- & -- & 0.793$\pm$0.003 & 0.829$\pm$0.004 & 0.889$\pm$0.004 & 0.631$\pm$0.005 \\
        \bottomrule
    \end{tabular}%
    }
\end{table*}

\vspace{2cm}

\begin{table*}[ht!]
    \centering
    \caption{Predictive performance on the \dataset{cars} dataset. Entries are reported as mean $\pm$ standard error over \(5\) seeds.}
    \label{tab:cars_results}
    \scriptsize
    \setlength{\tabcolsep}{3pt}

    \resizebox{\textwidth}{!}{%
    \begin{tabular}{lccccccccccc}
        \toprule
        Scenario & Metric & NPBoost & ANPBoost & NP & ANP & Task-ID GBT & GBT & GPLinear & LME & TabICL\_T & TabICL\_P \\
        \midrule
        within-task & RMSE & \textbf{0.257}$\pm$0.002 & 0.258$\pm$0.003 & 0.266$\pm$0.003 & 0.270$\pm$0.003 & 0.271$\pm$0.003 & 0.354$\pm$0.002 & 0.297$\pm$0.003 & 0.318$\pm$0.003 & 0.303$\pm$0.003 & 0.316$\pm$0.002 \\
        within-task & CRPS & 0.127$\pm$0.001 & 0.134$\pm$0.006 & \textbf{0.125}$\pm$0.001 & 0.128$\pm$0.001 & -- & -- & 0.138$\pm$0.001 & 0.174$\pm$0.001 & 0.134$\pm$0.001 & 0.151$\pm$0.001 \\
        few-shot & RMSE & 0.272$\pm$0.010 & \textbf{0.265}$\pm$0.008 & 0.281$\pm$0.009 & 0.288$\pm$0.012 & 0.401$\pm$0.018 & 0.367$\pm$0.011 & 0.368$\pm$0.023 & 0.373$\pm$0.035 & 0.456$\pm$0.025 & 0.387$\pm$0.013 \\
        few-shot & CRPS & 0.142$\pm$0.008 & \textbf{0.136}$\pm$0.002 & \textbf{0.136}$\pm$0.004 & 0.143$\pm$0.003 & -- & -- & 0.177$\pm$0.008 & 0.193$\pm$0.006 & 0.231$\pm$0.012 & 0.206$\pm$0.008 \\
        \bottomrule
    \end{tabular}%
    }
\end{table*}

\begin{table*}[ht!]
    \centering
    \caption{Predictive performance on the \dataset{Spotify} dataset. Entries are reported as mean $\pm$ standard error over \(5\) seeds. The within-task CRPS for GPLinear is unavailable because NA predictive standard deviations occurred for all seeds.}
    \label{tab:spotify_results}
    \scriptsize
    \setlength{\tabcolsep}{3pt}

    \resizebox{\textwidth}{!}{%
    \begin{tabular}{lccccccccccc}
        \toprule
        Scenario & Metric & NPBoost & ANPBoost & NP & ANP & Task-ID GBT & GBT & GPLinear & LME & TabICL\_T & TabICL\_P \\
        \midrule
        within-task & RMSE & 0.0907$\pm$0.0007 & 0.0912$\pm$0.0006 & 0.0922$\pm$0.0007 & 0.0928$\pm$0.0007 & 0.0928$\pm$0.0008 & 0.0970$\pm$0.0004 & 0.1012$\pm$0.0005 & 0.1021$\pm$0.0010 & 0.1119$\pm$0.0006 & \textbf{0.0874}$\pm$0.0006 \\
        within-task & CRPS & 0.1393$\pm$0.0235 & 0.1662$\pm$0.0120 & 0.1242$\pm$0.0172 & 0.1295$\pm$0.0242 & -- & -- & -- & 0.0572$\pm$0.0005 & 0.0592$\pm$0.0002 & \textbf{0.0447}$\pm$0.0003 \\
        few-shot & RMSE & 0.0979$\pm$0.0043 & 0.0998$\pm$0.0047 & 0.0998$\pm$0.0043 & 0.1012$\pm$0.0044 & 0.1240$\pm$0.0042 & 0.1015$\pm$0.0033 & 0.1102$\pm$0.0040 & 0.1057$\pm$0.0035 & 0.1226$\pm$0.0038 & \textbf{0.0909}$\pm$0.0032 \\
        few-shot & CRPS & 0.2032$\pm$0.0034 & 0.2093$\pm$0.0026 & 0.1630$\pm$0.0085 & 0.1738$\pm$0.0189 & -- & -- & 0.0590$\pm$0.0022 & 0.0593$\pm$0.0017 & 0.0673$\pm$0.0020 & \textbf{0.0480}$\pm$0.0018 \\
        \bottomrule
    \end{tabular}%
    }
\end{table*}

\begin{table*}[ht!]
    \centering
    \caption{Predictive performance on the \dataset{cows} dataset. Entries are reported as mean $\pm$ standard error over \(5\) seeds.}
    \label{tab:cows_results}
    \scriptsize
    \setlength{\tabcolsep}{3pt}

    \resizebox{\textwidth}{!}{%
    \begin{tabular}{lccccccccccc}
        \toprule
        Scenario & Metric & NPBoost & ANPBoost & NP & ANP & Task-ID GBT & GBT & GPLinear & LME & TabICL\_T & TabICL\_P \\
        \midrule
        within-task & RMSE & \textbf{0.232}$\pm$0.007 & 0.234$\pm$0.010 & 0.245$\pm$0.006 & 0.245$\pm$0.006 & 0.266$\pm$0.011 & 0.303$\pm$0.007 & 0.239$\pm$0.005 & 0.272$\pm$0.006 & 0.263$\pm$0.004 & 0.239$\pm$0.006 \\
        within-task & CRPS & 0.161$\pm$0.018 & 0.176$\pm$0.015 & 0.137$\pm$0.004 & 0.165$\pm$0.019 & -- & -- & \textbf{0.130}$\pm$0.002 & 0.152$\pm$0.003 & 0.153$\pm$0.002 & \textbf{0.130}$\pm$0.003 \\
        few-shot & RMSE & 0.258$\pm$0.014 & 0.249$\pm$0.011 & 0.270$\pm$0.014 & 0.254$\pm$0.011 & 0.317$\pm$0.011 & 0.296$\pm$0.012 & 0.253$\pm$0.014 & 0.280$\pm$0.010 & 0.293$\pm$0.013 & \textbf{0.236}$\pm$0.012 \\
        few-shot & CRPS & 0.167$\pm$0.012 & 0.173$\pm$0.010 & 0.174$\pm$0.013 & 0.191$\pm$0.010 & -- & -- & 0.140$\pm$0.008 & 0.158$\pm$0.005 & 0.174$\pm$0.009 & \textbf{0.131}$\pm$0.006 \\
        \bottomrule
    \end{tabular}%
    }
\end{table*}

\begin{table*}[ht!]
    \centering
    \caption{Predictive performance on the \dataset{bikes} dataset. Entries are reported as mean $\pm$ standard error over \(5\) seeds.}
    \label{tab:bikes_results}
    \scriptsize
    \setlength{\tabcolsep}{3pt}

    \resizebox{\textwidth}{!}{%
    \begin{tabular}{lccccccccccc}
        \toprule
        Scenario & Metric & NPBoost & ANPBoost & NP & ANP & Task-ID GBT & GBT & GPLinear & LME & TabICL\_T & TabICL\_P \\
        \midrule
        within-task & RMSE & 0.287$\pm$0.009 & \textbf{0.282}$\pm$0.009 & 0.294$\pm$0.008 & 0.296$\pm$0.010 & 0.302$\pm$0.006 & 0.322$\pm$0.006 & 0.429$\pm$0.005 & 0.580$\pm$0.006 & 0.547$\pm$0.008 & 0.309$\pm$0.006 \\
        within-task & CRPS & 0.129$\pm$0.008 & 0.124$\pm$0.006 & 0.121$\pm$0.003 & 0.124$\pm$0.006 & -- & -- & 0.200$\pm$0.002 & 0.325$\pm$0.002 & 0.276$\pm$0.004 & \textbf{0.108}$\pm$0.001 \\
        few-shot & RMSE & 0.588$\pm$0.070 & 0.550$\pm$0.058 & 0.577$\pm$0.063 & 0.551$\pm$0.054 & 0.706$\pm$0.091 & 0.755$\pm$0.115 & 0.646$\pm$0.086 & 0.664$\pm$0.029 & 0.690$\pm$0.032 & \textbf{0.351}$\pm$0.028 \\
        few-shot & CRPS & 0.340$\pm$0.032 & 0.322$\pm$0.022 & 0.316$\pm$0.036 & 0.292$\pm$0.036 & -- & -- & 0.325$\pm$0.040 & 0.395$\pm$0.020 & 0.375$\pm$0.016 & \textbf{0.151}$\pm$0.010 \\
        \bottomrule
    \end{tabular}%
    }
\end{table*}

\section{Model Plots}
\label{app:model_plots}
For the \dataset{reference} experiment and the Steps 1D fixed effect, we display the best-performing model for seed 0 under both prediction settings for each model class.

Figure~\ref{fig:app_coverage_calibration} shows the calibration curves for all NP-based models. They exhibit very good calibration.
ANPBoost achieves the lowest Mean Absolute Calibration Error (MACE) in within-task, whereas in few-shot, NP achieves the lowest MACE.

For the dense prediction visualizations (Figures~\ref{fig:reference_dense_predictions_within_task}--\ref{fig:reference_dense_predictions_few_shot}), predictive intervals for NP-based models are obtained using $20$ latent samples and $20$ decoder samples, and then taking empirical quantiles. For TabICL, we use the prediction quantiles directly provided by the model.

\begin{figure}[!htbp]
    \centering
    \includegraphics[width=1.0\linewidth]{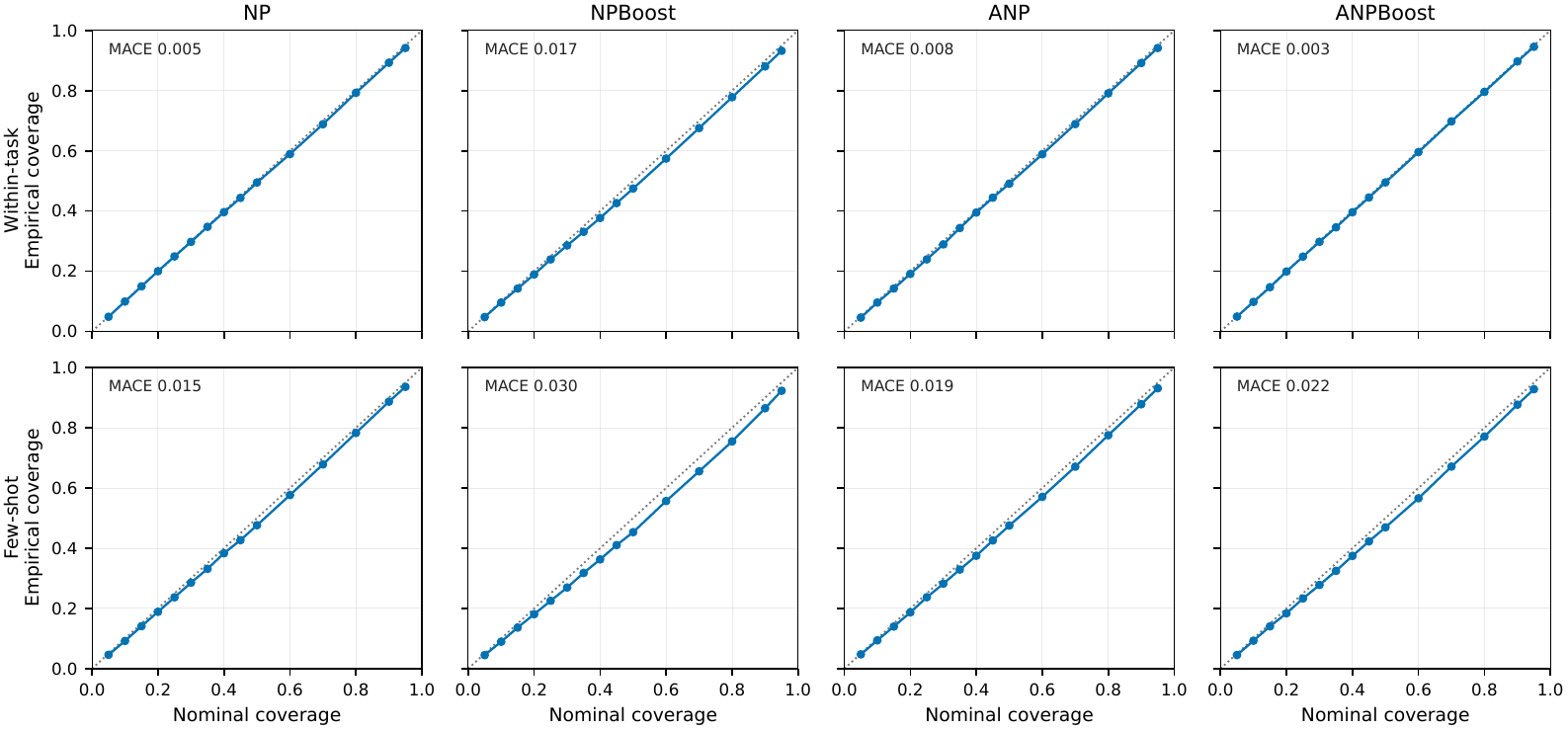}
    \caption{\textbf{Coverage calibration.} Predictive-interval coverage calibration for the one-dimensional \dataset{reference}  setting.}
    \label{fig:app_coverage_calibration}
\end{figure}

\begin{figure}[!htbp]
    \centering
    \includegraphics[width=1.0\linewidth]{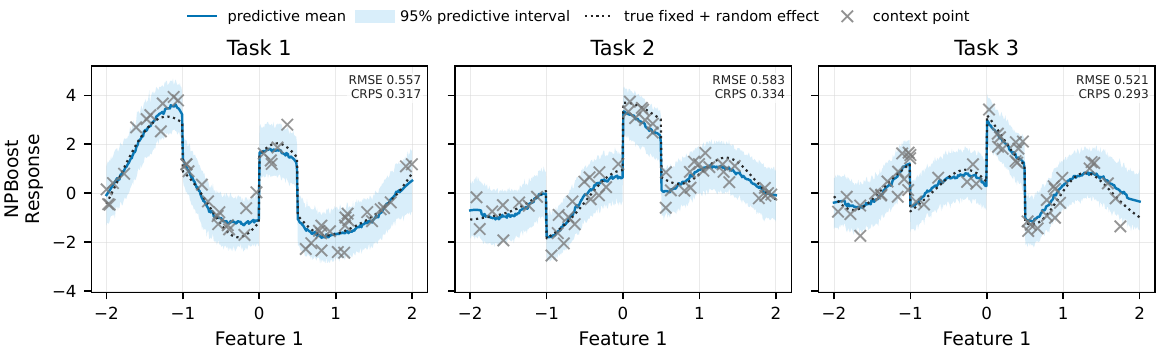}
    \caption{\textbf{Dense predictions (within-task).} Dense prediction plot for NPBoost for the one-dimensional \dataset{reference} setting in the within-task scenario.}
    \label{fig:reference_dense_predictions_within_task}
\end{figure}

\begin{figure}[!htbp]
    \centering
    \includegraphics[width=1.0\linewidth]{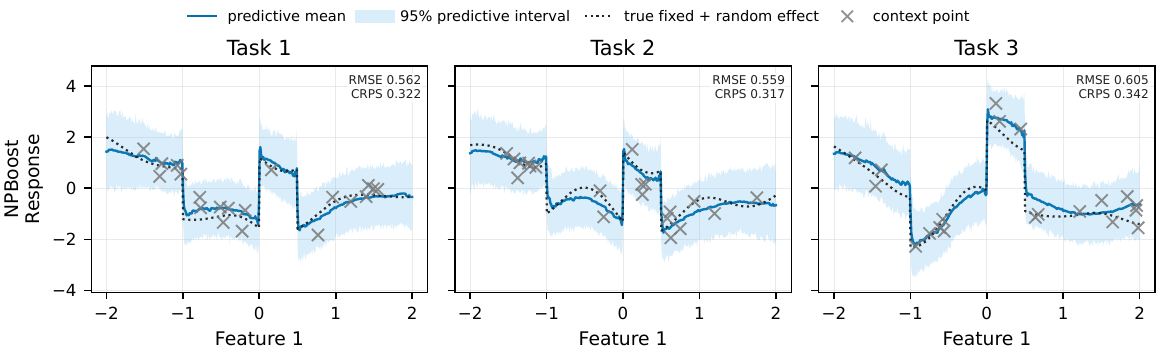}
    \caption{\textbf{Dense predictions (few-shot).} Dense prediction plot for NPBoost for the one-dimensional \dataset{reference} setting in the few-shot scenario.}
    \label{fig:app_dense_predictions_few_shot}
\end{figure}

\begin{figure}[!htbp]
    \centering
    \includegraphics[width=1.0\linewidth]{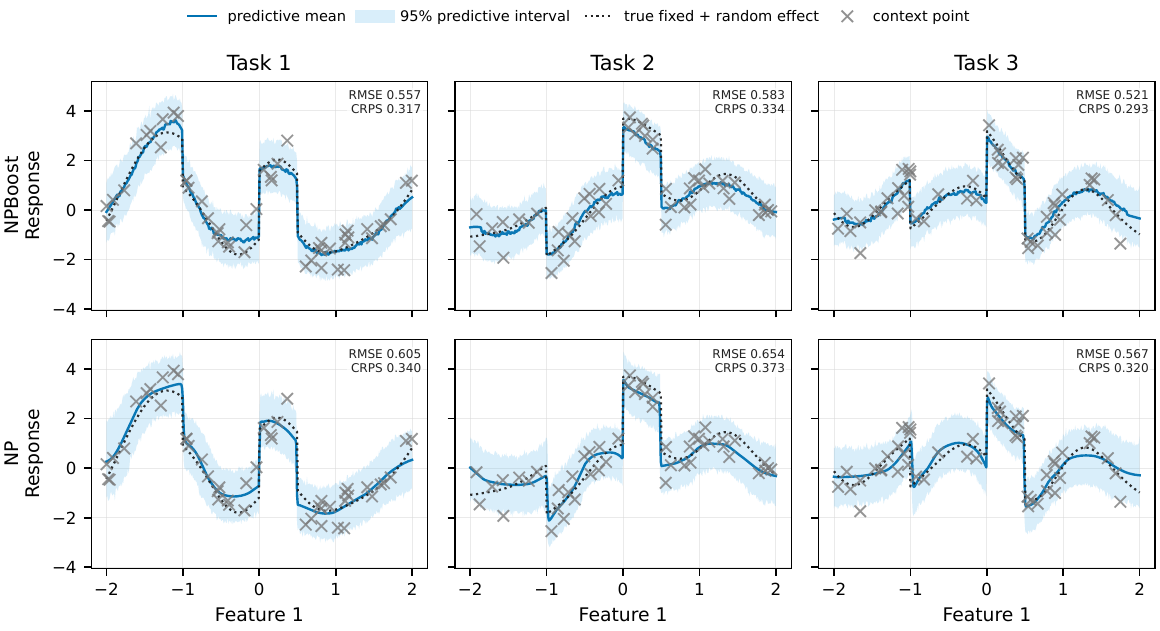}
    \caption{\textbf{Dense predictions (within-task).} Dense prediction plot for NPBoost and NP for the one-dimensional \dataset{reference} setting in the within-task scenario.}
    \label{fig:reference_dense_predictions_within_task_npboost_np}
\end{figure}

\begin{figure}[!htbp]
    \centering
    \includegraphics[width=1.0\linewidth]{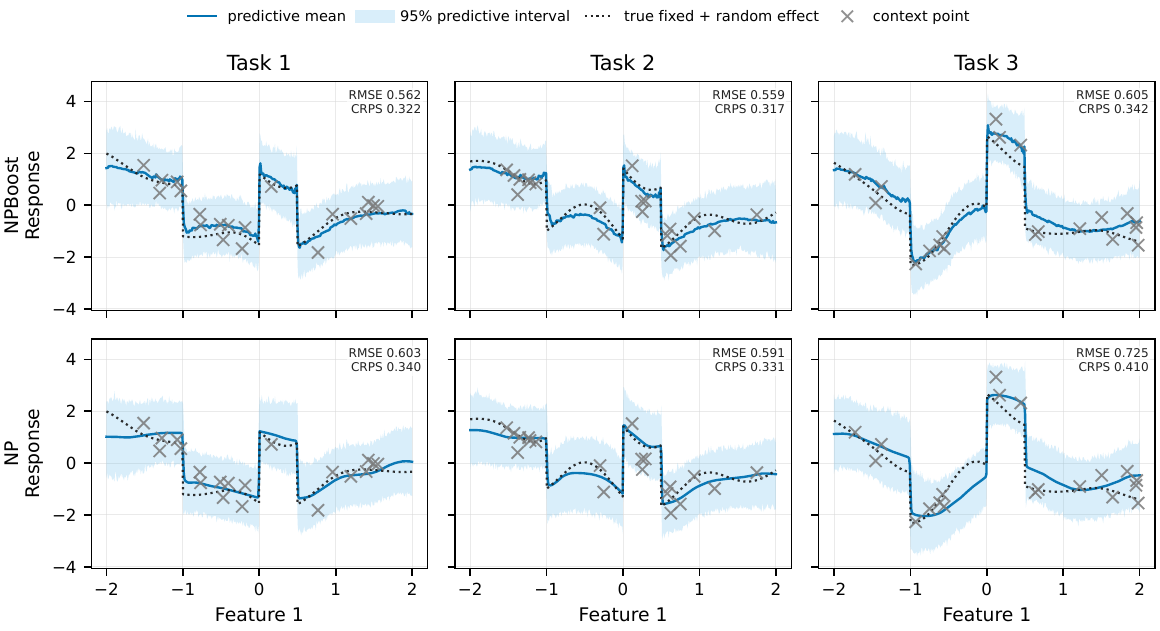}
    \caption{\textbf{Dense predictions (few-shot).} Dense prediction plot for NPBoost and NP for the one-dimensional \dataset{reference} setting in the few-shot scenario.}
    \label{fig:reference_dense_predictions_few_shot_npboost_np}
\end{figure}

\begin{figure}[!htbp]
    \centering
    \includegraphics[width=1.0\linewidth]{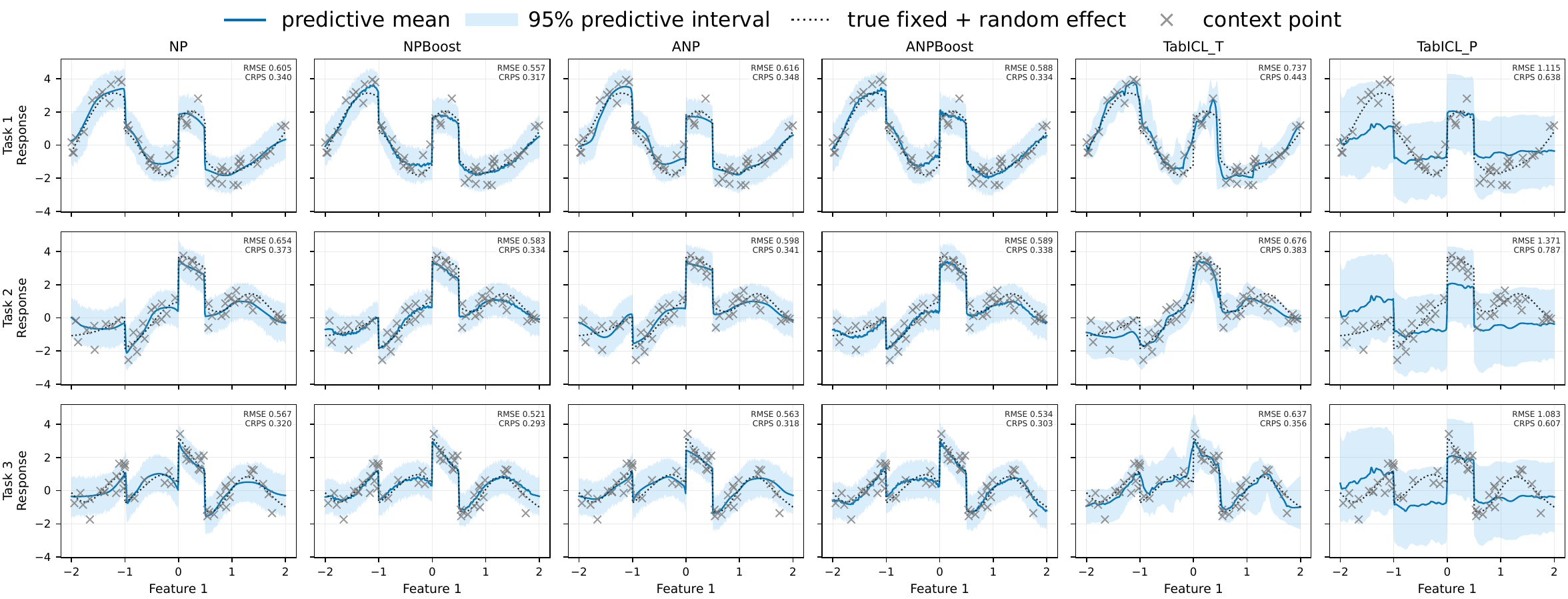}
    \caption{\textbf{Dense predictions (within-task).} Dense prediction plot for the one-dimensional \dataset{reference} setting in the within-task scenario.}
    \label{fig:app_dense_predictions_within_task}
\end{figure}

\begin{figure}[!htbp]
    \centering
    \includegraphics[width=1.0\linewidth]{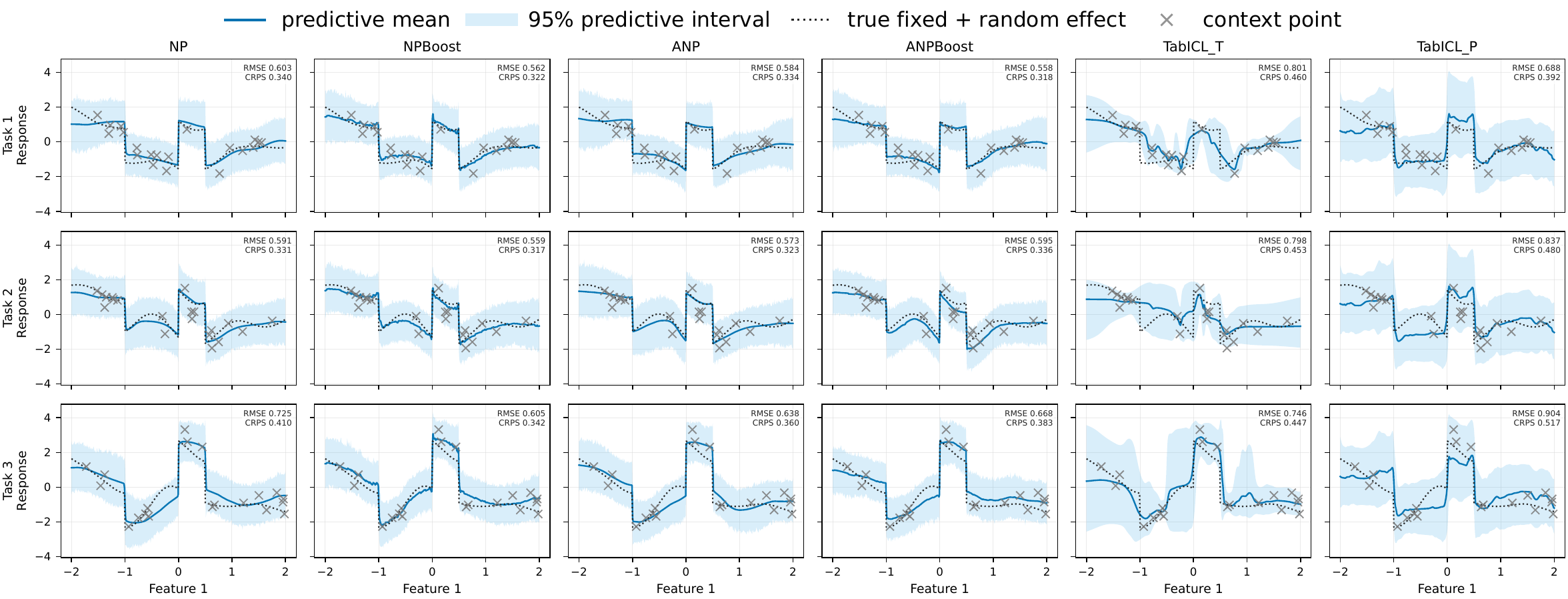}
    \caption{\textbf{Dense predictions (few-shot).} Dense prediction plot for the one-dimensional \dataset{reference} setting in the few-shot scenario.}
    \label{fig:reference_dense_predictions_few_shot}
\end{figure}

\clearpage
\section{Hardware Description}
\label{app:hardware}
\begin{table}[!htbp]
    \centering
    \renewcommand{\arraystretch}{1.2} 
    \small 
    \begin{tabular}{@{} l l >{\raggedright\arraybackslash}p{4.2cm} >{\raggedright\arraybackslash}p{4.2cm} l @{}}
        \toprule
        \textbf{Dataset} & \textbf{CPU Configuration} & \textbf{GPU Configuration} & \textbf{RAM} \\
        \midrule
        \dataset{reference} & 1x AMD EPYC 7763 (64 Cores) 2.45 GHz & 2x nVidia GeForce RTX 4090 (24 GB) & 512 GB \\
         \dataset{more-tasks}  & 1x AMD EPYC 9554 (64 Cores) 3.1 GHz & 2x nVidia RTX PRO 6000 Blackwell (96 GB) & 1152 GB \\
        \dataset{Brownian} & 1x AMD EPYC 7742 (64 Cores) 2.25 GHz & 2x nVidia GeForce RTX 3090 (24 GB) & 512 GB \\
        \dataset{Brownian-LN}  & 1x AMD EPYC 9554 (64 Cores) 3.1 GHz & 2x nVidia RTX PRO 6000 Blackwell (96 GB) & 1152 GB \\
        \dataset{irr-feat} & 1x AMD EPYC 7742 (64 Cores) 2.25 GHz & 2x nVidia GeForce RTX 3090 (24 GB) & 512 GB \\
        \dataset{cars}  & 1x AMD EPYC 9554 (64 Cores) 3.1 GHz & 2x nVidia RTX PRO 6000 Blackwell (96GB) & 1152 GB \\
        \dataset{Spotify} & 1x AMD EPYC 9554 (64 Cores) 3.1 GHz & 2x nVidia RTX PRO 6000 Blackwell (96 GB) & 1152 GB \\
        \dataset{cows} & 2x Intel Xeon Gold 6254 (18 Cores) 3.1 GHz & 2x nVidia GeForce RTX 2080 Ti (12 GB) & 512 GB \\
        \dataset{bikes} & 1x AMD EPYC 7763 (64 Cores) 2.45 GHz & 2x nVidia GeForce RTX 4090 (24 GB) & 512 GB \\
        \bottomrule
    \end{tabular}
    \caption{\textbf{Summary of experimental hardware setup.} All models with NP components and the TabICL experiments were run on GPU. The other models were run on CPU.}
    \label{tab:hardware_experiments}
\end{table}

\end{document}